\documentclass{article}

\usepackage{graphicx}
\usepackage{arxiv}
\usepackage{natbib}
\usepackage[normalem]{ulem}
\usepackage{bm}
\usepackage[utf8]{inputenc} 
\usepackage[T1]{fontenc}    
\usepackage{hyperref}       
\usepackage{url}            
\usepackage{xcolor}
\usepackage{wrapfig}

\usepackage{amsmath}
\usepackage{amsfonts}
\usepackage{amssymb}
\usepackage{amsthm}
\usepackage[mathscr]{eucal}
\usepackage{bbm}

\usepackage{algorithm, algpseudocode}
\usepackage{mathtools}
\usepackage{cleveref}
\usepackage{autonum} 
\usepackage{ esint }

\usepackage{tikz}
\usetikzlibrary{patterns}
\usetikzlibrary{decorations.pathreplacing,calligraphy}

\usepackage{enumitem}
\setlist[itemize]{leftmargin=1.5pc}

\theoremstyle{definition}
\newtheorem{assumption}{Assumption}
\newtheorem{remark}{Remark}
\newtheorem{definition}{Definition}

\newtheorem{theorem}{Theorem}
\newtheorem{proposition}{Proposition}
\newtheorem{lemma}{Lemma}

\newtheorem{claim}{Claim}

\newtheorem{example}{Example}

\colorlet{sgreen}{black!65!green}

\renewcommand{\P}{\mathbb{P}}

\newcommand{\thisHypName}{}
\newtheorem*{genericHyp}{\thisHypName}

\newcommand{\KLDiv}[2]{\mathcal{D}_{\text{kl}} \!\left(#1 | #2 \right)}

\newcommand{\Risk}{R}
\newcommand{\dist}{\mu}

\newcommand{\X}{\mathcal{X}}
\newcommand{\Y}{\mathcal{Y}}

\newcommand{\Hyp}{\mathcal{H}}

\newcommand{\E}{\mathcal{E}}

\newcommand{\real}{{\rm I\!R}}

\newcommand{\exf}[2]{\mathbb{E}_{#1}\left[#2\right]}

\newcommand{\expec}{\mathbb{E} }
\newcommand{\expecf}[1]{\mathbb{E}_{#1} \,}

\newcommand{\abs}[1]{\left|#1\right|}
\newcommand{\paren}[1]{\left(#1\right)}
\newcommand{\braces}[1]{\left\{#1\right\}}
\newcommand{\brackets}[1]{\left [ #1 \right ]}

\newcommand{\prob}{\mathbb{P}}

\newcommand{\Q}{Q}

\newcommand{\np}{{n_P}}
\newcommand{\nq}{{n_Q}}

\usepackage{dsfont}
\DeclareSymbolFont{bbold}{U}{bbold}{m}{n}
\DeclareSymbolFontAlphabet{\mathbbold}{bbold}

\DeclareMathOperator*{\argmax}{argmax}
\DeclareMathOperator*{\argmin}{argmin}

\DeclareMathOperator*{\Prob}{\prob}

\newcommand{\hstar}{h^{\!*}}

\def\defeq{\doteq}

\def\emRisk{\hat{\Risk}}

\newcommand{\indic}[1]{\mathbbm{1}\paren{#1}}

\newcommand{\classComp}[1]{\mathcal{C}(#1)}

\newcommand{\genGap}[3]{A(#1, #2, #3)}

\newcommand{\sgn}[1]{\mbox{sign}\left(#1\right )}
\newcommand{\size}[1]{\lvert #1 \rvert}

\def\RLHyp{{\Tilde{\Hyp}}}

\def\den{\mathrm{f}}

\def\functionClass{\mathcal{F}}

\def\event{E}

\newcommand{\blue}[1]{\textcolor{blue}{#1}}

\def\hhat{\hat{h}}
\def\hHyp{\hat{\Hyp}}

\def\Li{L_{in}}
\def\Lo{L_{out}}
\def\Ri{R_{in}}
\def\Ro{R_{out}}

\def\lv{\frac{1}{256} \cdot \paren{\frac{1}{\np}}^{1/\rho_a}}
\def\hdist{\delta}

\def\el{\epsilon_L}
\def\elin{\epsilon_{\Li}}
\def\er{\epsilon_R}
\def\erin{\epsilon_{\Ri}}

\newenvironment{proofof}[1]{\noindent\textbf{Proof of {#1}}
 \hspace*{1em} }{$\blacksquare$ \bigskip }

\newsavebox{\savepar}

\title{Limits of Model Selection under Transfer Learning}

\author{
 Steve Hanneke \\
  Purdue University\\
  \texttt{steve.hanneke@gmail.com} \\
   \And
 Samory Kpotufe \\
  Columbia University \\
  \texttt{samory@columbia.edu} \\
  \And 
  Yasaman Mahdaviyeh\\
  Columbia University\\
  \texttt{yasamanmdv@cs.columbia.edu}\\
}

\begin{document}

\maketitle

\begin{abstract}%
Theoretical studies on \emph{transfer learning} (or \emph{domain adaptation}) have so far focused on situations with a known hypothesis class or \emph{model}; however in practice, some amount of model selection is usually involved, often appearing under the umbrella term of \emph{hyperparameter-tuning}: for example, one may think of the problem of \emph{tuning} for the right neural network architecture towards a target task, while leveraging data from a related \emph{source} task.  

In addition to the usual tradeoffs on approximation vs.\ estimation errors involved in model selection, this problem brings in a new complexity term, namely, the \emph{transfer distance} between source and target distributions, which is known to vary with the choice of hypothesis class. 

We present a first study of this problem, focused on classification. Remarkably, the analysis reveals that \emph{adaptive rates}, i.e., those achievable with no distributional information, can be arbitrarily slower than \emph{oracle rates}, i.e., when given knowledge on \emph{distances}.

\end{abstract}


\section{Introduction}
Domain adaptation or Transfer learning concern settings where data from a \emph{source} distribution $P$ is to be leveraged to improve learning on a target distribution $Q$ where perhaps less data is available. While this problem has received much renewed attention of late, theoretical studies have focused on settings where a suitable hypothesis (or model) class $\Hyp$ is already known. However, this is rarely the case in practice where some amount of model selection is required, as often referred to as \emph{hyperparameter tuning}: one wishes, e.g., to tune for the right architecture with neural networks, a suitable polynomial degree in regression, or an appropriate kernel for kernel machines, all while leveraging both source and target data. {Importantly, as target data is often limited in these settings, it ideally should not be used alone to drive model selection, even though it is a priori unclear how to leverage the source data.} 

We present a first study of this problem, in the context of classification, under a simple formalism where we assume a hierarchy of models $\braces{\Hyp_i}, \Hyp_i \subset \Hyp_{i+1}$, each with known complexity $d_i$ (here VC-dimension); the problem is then to try and understand the achievable target $Q$-risk in modern transfer settings with access to both source and target data, as opposed to just target data. { We note however that our analysis allows for no target data, as in fact we have no restriction on data sizes from either source nor target.} 

To establish a baseline performance, assume the hierarchy $\braces{\Hyp_i}$ admits a global $Q$-risk minimizer $\hstar_{Q}$ from an unknown model $\Hyp_{i^*_Q} \in \braces{\Hyp_i}$. Then it is known that, using $n_Q$ data from $Q$, an excess risk $\E_Q(\hat h) \doteq \E_Q(\hat h; \hstar_Q) \lesssim \sqrt{d_{i^*_Q}/n_Q}$ is achievable without prior knowledge of $\Hyp_{i^*_Q}$, e.g., via structural risk minimization (SRM), a.k.a., complexity regularization, which essentially tradeoff \emph{estimation error} $\sqrt{d_{i}/n_Q}$, and \emph{approximation error} $\min_{h\in \Hyp_i} \E(h, h^*_Q)$ over models $\braces{\Hyp_i}$. 

Now, model selection in a transfer scenario, i.e., given related source data from $P$, involves an additional tradeoff parameter: the \emph{distance} or information that $P$ yields on $Q$, which is now well understood to be tied to the choice of hypothesis class $\Hyp_i$. Early notions of distance $P\to \Q$, e.g., from seminal works of \cite{mansour2009domain, ben2010theory} already formalize the idea that the differences between $P$ and $Q$ are only relevant in regions of space in line with $\Hyp$, e.g., disagreement regions between given hypotheses in $\Hyp$. In other words, while a model choice $\Hyp$ out of the hierarchy $\braces{\Hyp_i}$ may balance estimation and approximation errors, it may fail to maximally leverage the data from $P$ if it induces a large \emph{distance} $P\to Q$. 

As the distances $P\to Q$ induced over models in $\braces{\Hyp_i}$ are a priori unknown (however formalized), our analysis especially distinguishes betwen \emph{adaptive} model selection rates---i.e., rates achievable from $P$ and $Q$ samples alone without distributional information---and usual minimax \emph{oracle} rates. Remarkably, unlike in usual model selection, these can be significantly different.

\paragraph{Main Results.} For a fixed class $\Hyp$, we adopt a recent notion of \emph{distance} $P\to Q$ from \citep{hanneke2019value} comprised of two components: (1) the excess risk $\E_Q(\hstar_P)$ of a risk minimizer $\hstar_P$ under $P$, and (2) a \emph{transfer-exponent} $\rho$ which essentially measures the effective sample size contributed by $P$ to the target problem $Q$. Thus suppose access to $n_P$ samples from $P$ and $n_Q$ samples from $Q$, the following upper-bound was shown to be achievable adaptively: 

\begin{align} \label{eq:transfer}
\E_Q(\hat h) \lesssim \min \braces{\paren{{d}/{n_P}}^{1/2\rho} + \E_Q(\hstar_P)\ ; \ \paren{{d}/{n_Q}}^{1/2}}, \text{ where } d \text{ is the VC dimension of } \Hyp. 
\end{align} 
For sanity check, note that \eqref{eq:transfer} is of order $(d/(n_P+n_Q))^{1/2}$ when $P=Q$, i.e., $\rho =1$, $\E_Q(\hstar_P)=0$. Also notice that the rate is faster with smaller $\rho$ and $\E_Q(\hstar_P)$. 

Now, if we knew the above rate to be tight in general, we then get a first sense of the best rates we might expect for any fixed model choice $\Hyp_i$ out of the hierarchy.

\noindent $\bullet$ \emph{Tightness of \eqref{eq:transfer}}. As a first basic result, we show that the above adaptive rate on a fixed choice $\Hyp$, admits matching lower-bounds over any parameter value (Theorem \ref{thm: basic lower bound}). This complements a lower-bound of \citep{hanneke2019value} which only holds for $\E_Q(h^*_P) = 0$. This is especially important in our setting in order to cover a rich variety of situations.  

\noindent $\bullet$ \emph{Adaptive Upper-Bounds and Speedups.} Now suppose that each $\Hyp_i$ in the hierarchy admits \emph{transfer distance} $(\rho_i, \E_Q(\hstar_{P, i}))$, a priori unknown. Together with known class complexity $d_i$, and sample sizes $n_p$, $n_Q$, these distances induce subtle tradeoffs on model choices $\Hyp_i$ for the $Q$ task. 

--- First, we verify through some technical examples, namely basic neural-networks, that indeed some rich set of tradeoffs are captured through the above parametrization. That is, rich combinations of $(\rho_i, \E_Q(\hstar_{P, i}))$ emerge from the interaction between $(P, Q)$ and nested network architectures. 

--- Having established the tightness of the above equation \ref{eq:transfer}, and given the baseline of model selection under target, we can show (see Lemma \ref{lem: single level bounds.}) that selecting any fixed $\Hyp_i$ would yield an adaptive upper-bound of 
\begin{align} 
\E_Q(\hat h_i) \lesssim \phi(i), \text{ where } 
\phi(i)\approx \min \braces{\paren{{d_i}/{n_P}}^{1/2{\rho_i}} + \E_Q(\hstar_{P, i})\ ; \ \paren{{d_{i^*_Q}}/{n_Q}}^{1/2}}.
\end{align}

Unfortunately, as we discussed in the next bullet point, no algorithm exist that can minimize $\phi(i)$ in general and achieve optimal tradeoff on \emph{distance}. Instead, we establish the following adaptive guarantee (see Theorem \ref{thm: upper bound}). Suppose that $\braces{\Hyp_i}$ admits a global $P$-risk minimizer $\hstar_P$ at unknown level $i^*_P$; then there exists a procedure $\hat h$ achieving  
\begin{align} 
\E_Q(\hat h) \lesssim \phi(i^*_P) \text{ from samples alone}. 
\end{align}
In other words, the procedure automatically favors model selection under source $P$---at least commensurate with the unknown model $\Hyp_{i^*_P}$---if $P$ is thus informative on $Q$, and falls back on leveraging target data otherwise, all without prior knowledge of distributional parameters. 


\emph{We emphasize that in contrast, popular \emph{SRM} approaches yield no clear such guarantee:} suppose $n_Q =0$, SRM can only guarantee low $P$-risk, but no specific choice of model class.  

\noindent $\bullet$ \emph{Oracle Rates are Unachievable.} With knowledge of distance parameters $\braces{(\rho_i, \E_Q(\hstar_{P, i}))}$ (or at least of the ranking they induce on $\braces{\phi(i)}$), an \emph{oracle} procedure can achieve the rate $\min_i \phi(i)$, which can be arbitrarily faster than $\phi(i^*_P)$. 

Interestingly, as we show in Theorem \ref{thm: simple rho lower bound }, no \emph{adaptive} procedure, i.e., without such domain knowledge, can achieve a bound better than $\phi(i^*_P)$, without further structural conditions on the hierarchy $\braces{\Hyp_i}$, even in situations where $\min_i \phi(i)\ll \phi(i^*_P)$. This result holds even when the learner $\hat h$ is \emph{improper}, i.e., when $\hat h$ is allowed to return a hypothesis outside of $\cup_i \Hyp_i$. 

\paragraph{Related Work.}
Transfer Learning has received much attention over the years, with studies, both in the context of classification and regression, considering various notions of relations between $P$ and $Q$. Early works include \citep{ben2007analysis, crammer2008learning, cortes2008sample, ben2010theory, gretton2009covariate, mansour2009multiple} which already recognize the importance of the choice of hypothesis class in quantifying the information the source $P$ has on the target $Q$. These ealrier works have been refined over time, e.g., considering multiple source distributions rather than just one \citep{maurer2013sparse,pentina2014pac,yang2013theory, maurer2016benefit}. 

More recently, \emph{assymetric} notions of discrepancy have been proposed, noting that $P$ may have information on $Q$ but not the other way around \citep{kpotufe2018marginal, hanneke2019value, achille2019information, mousavi2020minimax}. We adopt such a notion in this work. 

Despite much of the attention on this problem, a single hypothesis class $\Hyp$ has been commonly assumed. However, a separate line of work on \emph{meta-learning} can be seen as somewhat related, as they often assume relationship between optimal predictors, often in the form of a shared \emph{low-dimensional substructure}; these settings may be recast as learning a target hypothesis class of lower complexity \citep{ando2005framework, muandet2013domain, mcnamara2017risk, arora2019theoretical, jalali2010dirty,lounici2011oracle, negahban5773043, du2020few,tripuraneni2020theory}. This however does not embody the full richness of model selection. 

\paragraph{Paper Organization.} We start in Section \ref{sec:prelim} with basic definitions and setup. This is followed by an overview of results in Section \ref{sec:overview}. Much of the proofs are discussed in Section \ref{sec:analysis} with more technical results relegated to the Appendix.

\section{Preliminaries}\label{sec:prelim}
\subsection{Setup}
\paragraph{Basic Definitions.} 
Let $X, Y$ be jointly distributed according to some measure $\mu$ (later $P$ or $Q$), 
where $X$ is in some domain $\X$ and $Y\in \Y \doteq \braces{\pm 1}$. A \emph{hypothesis class} or \emph{model} is a set $\Hyp$ of functions $\X \mapsto \Y$. All these objects are assumed to be measurable, so that we may consider classification risks of the form 
$R(h) \doteq \expec [h(X) \neq Y]$, as measured under $\mu$.

\begin{definition} The {\bf excess risk} of a classifier, w.r.t. $\Hyp$, is defined as $\E(h) = R(h) - \inf_{h'\in \Hyp} R(h')$. 

Furthermore we use the notation $\E(h, h')\doteq R(h) - R(h') = \E(h) - \E(h')$. 
\end{definition} 

We adopt the following classical noise conditions (see e.g. \citep{MN:06,koltchinskii:06,bartlett:06b}). 

\begin{definition} \label{assmp: Bernstein noise condition}
Assume $R(h)$ is minimized at $\hstar\in \Hyp$. We say that $\Hyp$ satisfies a {\bf Bernstein Class Condition} (BCC), as measured under $\mu$, with parameters $(C_\beta, \beta$),$C_\beta > 0$ and $\beta \in [0, 1]$, if $\forall h \in \Hyp$
\begin{equation}
    \Prob (h \neq \hstar) \le C_\beta\cdot \E (h)^{\beta}. 
\end{equation}
\end{definition}
Note that the condition trivially holds for $\beta = 0$, $C_\beta = 1$.
The condition captures the hardness of the learning problem: when $\beta =1$, which formalizes \emph{low noise} regimes, we expect fast rates of the form $n^{-1}$, in terms of sample size $n$, while for $\beta = 0$, rates are of the more common form $n^{-1/2}$. 

{ When $\hstar$ is not unique, BCC remains well defined (i.e., the definition is invariant to the choice of $\hstar$), as it  imposes (when $\beta > 0$) that all $\hstar$'s differ on a set of measure $0$ under the data distribution.} 

\paragraph*{Transfer Setting.} We consider a \emph{source} and \emph{target} distributions $P$ and $Q$ on $(X, Y)$, where we let $\E_P, \E_Q$ denote excess-risks under $P$ and $Q$. We are interested in excess risk $\E_Q(\hat h)$ of classifiers trained jointly on $n_P$ i.i.d samples from $P$, and $n_Q$ i.i.d. samples from $Q$. Achievable such excess risks necessarily depend on the \emph{distance} $P\to Q$ appropriately formalized.  

We adopt some recent notion of \emph{distance} from \citep{hanneke2019value}; for ease of exposition, we make the following simplifying assumptions. 

\begin{assumption} \label{assump:hstarinclass}
We assume for any $\Hyp$ considered henceforth that $\E_P$ and $\E_Q$ are minimized in $\Hyp$. We let $\hstar_P$, $\hstar_Q$ denote any such respective risk minimizers. 
Furthermore, if multiple minimizers $\braces{\hstar_P}$ exist under $P$, we assume that one of them achieves $\sup_{\braces{\hstar_P}} \E_Q(\hstar_P)$, and denote it $\hstar_P$. 
\end{assumption}



The distance $P\to Q$ is then given by $\E_Q(\hstar_P)$, and the following quantity $\rho$: 

\begin{definition}\label{def: transfer exponent model class} We call 
  $0<\rho\leq \infty$ {\bf transfer exponent} from $P$ to $Q$ with respect to $\Hyp$ if there exists $C_{\rho} > 0$ such that for all $h \in \Hyp$, 
\begin{equation}
    C_{\rho} \cdot \E_P(h, \hstar_{P})^{1/\rho} \ge  \E_Q(h, \hstar_{P}).
\end{equation}

We say that $\rho$ is  {\bf minimal } when no $ 0 < \rho' < \rho$ is a transfer  exponent from $P$ to $Q$ w.r.t. $\Hyp$.
\end{definition}

Notice that the above parametrization holds trivially for $\rho = \infty$, $C_\rho = 1$. Larger values of the pair $(\rho, \E_Q(\hstar_P))$ denote higher discrepancy $P\to Q$. For intuition on $\rho$, consider the case $\hstar_P = \hstar_Q=\hstar$; then $\rho$ simply describes how well $P$ reveals the \emph{decision boundary} defined by $\hstar$, i.e., whether hypotheses $h$ with small $P$-excess risk also have small $Q$-excess risk. Various examples of the continuum $\rho \to \infty$ are given in \citep{hanneke2019value, hanneke2022no}. We build on the intuition therein to derive Examples \ref{ex:NNthresholds} and \ref{ex:NN-Relu} of Section \ref{sec:examples} below for our specific setting with a hierarchy of hypothesis classes. 

\paragraph{Model Selection Setting.} We consider a situation where the learner has access to a hierarchy $\braces{\Hyp_i}, \Hyp_i \subset \Hyp_{i+1}$ of hypothesis classes, where each $\Hyp_i$ has VC dimension $d_i$, $d_i \leq d_{i+1}$. We let $\hstar_{P,i}$, $\hstar_{Q, i}$ denote the $P$ and $Q$ risk minimizers over model $\Hyp_i$ (according to Assumption \ref{assump:hstarinclass}). 

\begin{assumption} 
We assume $\braces{\Hyp_i}$ admits global risk minimizers $\hstar_P$ and $\hstar_Q$ w.r.t $P$ and $Q$; let $i^*_P$, $i^*_Q$, unknown to the learner, denote the indices of the smallest classes containing an $\hstar_P$, resp. $\hstar_Q$. 
\end{assumption}

\begin{definition}[Noise and Transfer Parameters]\label{def:noiseandtransfer}
We let $(C_{\beta_{P, i}}, \beta_{P, i})$ and $(C_{\beta_{Q, i}}, \beta_{Q, i})$ denote BCC parameters for $\Hyp_i$ w.r.t. $P$ and $Q$. For simplicity, we let $ \beta_P \defeq \beta_{P, i^*_P}$ and $\beta_Q \defeq { \inf_{i \geq i^*_Q}\beta_{Q, i}}$. 

Finally, we let $(C_{\rho_i}, \rho_i)$ denote transfer-exponents from $P$ to $Q$ under class $\Hyp_i$. 
\end{definition}

\begin{remark}
It remains unclear from our analysis whether and when we could achieve an adaptive model-transfer rate in terms of $\beta_{Q, i^*_Q}$ rather than $\beta_Q$ as defined above. The reason for our definition becomes apparent in the proof of \Cref{lem: single level bounds.} when considering whether to bias towards $P$ or $Q$ in model selection while $i^*_Q$ is unknown. 

The above definition of $\beta_Q$ however remains general as it admits the most common noise conditions in the literature, e.g., $\beta_{Q, i} =0$ (leading to usual $\sqrt{n}$ convergence results), and Tsybakov and Massart's noise conditions whenever the Bayes classifier is in the class. Namely (Tsybakov's noise condition) suppose there exists $0 < \alpha \leq \infty$ and $C_\alpha > 0$ such that for all $ t> 0$
\begin{align}
    Q_X \{x: \abs{\expec_Q \brackets{Y \mid X=x}} \le t \}\le C_\alpha t^\alpha.
\end{align}

If $\hstar_Q$ is the Bayes classifier, then BCC holds under $Q$ for all levels $i \ge i^*_Q$ with $ \beta_{Q, i} = \frac{\alpha}{1 + \alpha}$ and  $C_{\beta_{Q, i}}$ being some function of $\alpha$ and $C_\alpha$. See Proposition 1 of \citet{tsybakov:04}.

\end{remark}

\begin{assumption}
We assume for simplicity that all $C_{\beta_{P, i}}, C_{\beta_{Q, i}}$ are upper-bounded by some $C_\beta$. 
\end{assumption}


\subsection{Examples and Intuition on Tradeoffs.} \label{sec:examples}
We start with the following remark. 

\begin{remark}[Implicit Structure on $\braces{(\rho_i, \E_Q(\hstar_{P, i}))}$] 
To get some intuition, let's consider a simpler situation where $\hstar_{P, i}$ is unique for each class $\Hyp_i$.  
It then follows by definition, and the fact that the classes are nested, that for $i> i^*_P$, we have that $\rho_i$ is also a transfer-exponent for $i^*_P$. 
Also, by Assumption \ref{assump:hstarinclass}, for $i > i^*_P$, $\hstar_{P, i} = \hstar_P$ so we have $\E_Q(\hstar_{P, i}) = \E_Q(\hstar_{P, i^*_P})$. 

In other words, \emph{model selection would not favor $\Hyp_i$ over $\Hyp_{i^*_P}$} if $i>i^*_P$. However, for $i<i^*_P$, the distance parameters $(\rho_i, \E_Q(\hstar_{P, i}))$ are unrestricted---i.e., either term may increase or decrease as $i$ increases to $i^*_P$---if we impose no further condition on the hierarchy $\braces{\Hyp_i}$, thus inducing subtle tradeoffs. Such unrestricted increase or dicrease in distance below $i^*_P$ is illustrated by the examples below and further by the lower-bound construction for Theorem \ref{thm: lower bound for adaptivity to erisk across levels}. 
\end{remark}

Note that, similarly, for $\mu$ denoting either $P$ or $Q$, the BCC parameters $\beta_{\mu, i}$'s are nondecreasing for $i \geq i^*$. Thus, following from the remark, suppose for instance that the \emph{distances} $(\rho_i, \E_Q(\hstar_{P, i}))$ were decreasing with $i = 1, 2, \ldots i^*_P$, either in the first or second terms. Then, while usual model selection (as in a non-transfer setting) would favor the smallest class with small error, now it could be that a larger class transfers better. On the flip side, we could have situations were all $\rho_i$'s increase, while $\E_Q(\hstar_{P, i})$'s decrease, leading to similarly complicated tradeoffs.  

 The examples below illustrate such richness of situations in the case of simple two-layer neural networks for $X\in \real$, where the nested classes $\Hyp_i\subset \Hyp_{i+1}$ correspond to increasing width. We emphasize that the main point of these examples is to illustrate the basic thesis that \emph{distance} between source $P$ and target $Q$ may change with given classes in the hierarchy, in particular for model classes that speak to contemporary interest. We will revisit some such examples in Section \ref{sec:mainupperbounds} when discussing achievable bounds.

\begin{example}[Two Layer Neural Nets with Threshold Activation]\label{ex:NNthresholds}
Define $\Hyp_i = \braces{h_\theta: \real \mapsto \pm 1}$, indexed over $\theta \defeq (i, a,r, w, b)$, for $a, w, b \in \real^i$ and $r \in \real$, and where $h_\theta$ is of the form 
\begin{equation} \label{eq: def two layer threshold}
     h_\theta (x) \doteq \sgn{\sum_{j=1}^{i} a_j \; \sgn{w_j x - b_j} + r}.
\end{equation}

\begin{proposition} \label{prop: threshold neural net}
  For every finite sequence $1 \le \rho_1 \le \dots \le \rho_L$, there exists source and target distributions $P$, $Q$ over $[0, 1] \times \{-1, +1\}$ 
such that $\forall 1 \le i \le L $, $\rho_i$ is the minimal transfer exponent from $P$ to $Q$ w.r.t. $\Hyp_i$, $i^*_P = i^*_Q = L$ and $\beta_{P, i} = \beta_{Q, i} = 1$. Furthermore, the sequence of values $\E_Q(\hstar_{P, i})$, $i = 1, 2, \ldots, L$, is strictly decreasing,  depends only on $L$, but not on $\braces{\rho_i}$; finally we have that $C_{\rho_i}$s are upper and lower bounded by functions that depend on $i$ and $L$ only, but not on the choice of $\braces{\rho_i}$. 
\end{proposition}

\emph{In particular, as we may have $\rho_i$'s increasing while $\E(\hstar_{P, i}$ decrease, we see that nontrivial tradeoffs may indeed occur in practice.} 
\end{example}

\begin{example}[Two Layer (Residual) Neural Net with Relu Activation]\label{ex:NN-Relu}
Let $\RLHyp_i \doteq \braces{h_\theta: \real \to \pm 1}$ indexed over $\theta \defeq (i, a, r,   w, b,  \alpha)$, for $a, w, b \in \real^i$ and $r, \alpha \in \real$, and where $h_\theta$ is of the form 
\begin{equation} \label{eq: def two layer relu}
    h_\theta(x) = \sgn{\left(\sum_{j=1}^i
        a_j [w_jx +b_j]_+ 
    \right)  + \alpha x + r}, \text{ using the notation } [\cdot]_+ \defeq \max ({0, \cdot}).
\end{equation}

 \end{example} 
 
Next proposition uses results of \cite{Aliprantis2006} to connect ReLu residual neural nets to threshold neural nets in one dimension.

\begin{proposition} \label{prop: relu neural net}
    Let $\RLHyp_i$ be the class of Relu neural nets of Example \ref{ex:NN-Relu}, and let $\Hyp_i$ be the class of neural nets from Example \ref{ex:NNthresholds}.  We have $\RLHyp_i = \Hyp_{i+ 1}$, and consequently, \Cref{prop: threshold neural net} still holds. 
\end{proposition}

The proofs of the propositions above are given in \Cref{sec: main example proofs}. In particular, the proof of Proposition \ref{prop: threshold neural net} illustrates how the behavior of $P$ and $Q$ around decision boundaries (defined by optimal classifiers at each level $\Hyp_i$) affects  model-transferability; as such, even though for simplicity we focus on $X\in \real$ for these examples, the same insights extend to $\real^d$.

\section{Overview of results}\label{sec:overview}

For intuition behind the analysis, we start with trying to understand \emph{adaptive} transfer rates at a single level $\Hyp_i$ of the hierarchy. A result of \citep{hanneke2019value} (see Proposition 2 therein) offers a first glimpse. It states roughly that, for a fixed class $\Hyp$, there exists an adaptive $\hat h$ with access to $n_P$ samples from $P$ and $n_Q$ samples from $Q$, such that, w.h.p.
\begin{align} 
\E_Q(\hat h) \lesssim \min \left \{\left(\frac{d}{n_P} \right)^{\frac{1}{(2- \beta_P)\rho}} + \E_Q(\hstar_P), \left( \frac{d}{n_Q} \right)^{\frac{1}{2 - \beta_Q}}\right \}, \label{eq:transfer2}
\end{align} 
where $\beta_P, \beta_Q$ denote BCC parameters for $P$ and $Q$. While they show that this is tight (for all $\rho, \beta_P, \beta_Q$), their construction assumes $\E_Q(\hstar_P)=0$, which is too restrictive in our setting. 

We start our analysis by first showing that \eqref{eq:transfer2} is indeed tight in all parameters.  

\subsection{Lower Bound for a Fixed $\Hyp$}
We consider the following class of pairs of distributions $P, Q$ w.r.t. a fixed $\Hyp$. 

\begin{definition}[$\Xi$ class] Let $\Hyp$ denote a hypothesis class, and let $\beta_P, \beta_Q \in [0, 1), \rho>0, \alpha < 1$. 
We then define $\Xi = \Xi(\Hyp, \beta_P, \beta_Q, \rho, \alpha)$ as the set of pairs of distributions $(P, Q)$ satisfying the following conditions. (i) Assumption \ref{assump:hstarinclass} holds, (ii) both $P, Q$ satisfy a BCC with respective parameters $(1, \beta_P)$, $(1, \beta_Q)$ (iii) $\rho$ is a transfer exponent $P$ to $Q$ w.r.t. $\Hyp$, with $C_\rho \le  1$, and (iv) $\E_Q(\hstar_P) \le \alpha$. 
    
\end{definition}
\begin{theorem}\label{thm: basic lower bound}
    Fix some hypothesis class $\Hyp$ with VC dimension $d \ge 9$. Pick any $\rho \ge 1$, and $ \beta_P, \beta_Q \in [0,1)$ and let $\Xi$ denote the corresponding class. For every $n_P, n_Q$ where $\max  \{ n_P, n_Q \} > d$, let $\hat{h}$ be any classifier that has access to $n_P$ and $n_Q$ source and target samples. Then, there exists a universal constant $ c>0$ s.t. 
\begin{equation}
    \sup_{(P, Q) \in \Xi} \ 
        \Prob_{P^{n_P} \times Q^{n_Q}} 
            \left[ \E_Q(\hat{h}) \ge c\cdot
        \min \left \{\left(\frac{d}{n_P} \right)^{\frac{1}{(2- \beta_P)\rho}} + \alpha, \left( \frac{d}{n_Q} \right)^{\frac{1}{2 - \beta_Q}}\right \}\right]
        \ge \frac{3 - 2 \sqrt{2}}{8}.
\end{equation}
\end{theorem}

The result extends a lower-bound construction of \citep{hanneke2019value} by \emph{randomizing} the relation between a fixed $\hstar_P$ and candidates $\hstar_Q$'s. The proof is given in \Cref{app sec: lower bound proofs}.

\subsection{Upper Bound}\label{sec:mainupperbounds}

Having established the tightness of \eqref{eq:transfer2} over the range of parameters (except for $0< \rho < 1$), we now have a sense of the rates achievable if we fixed a level $\Hyp_i$. However, as we already know that, ignoring samples from source $P$, a baseline rate of $(d_{i^*_Q}/n_Q)^{({1}/{2-\beta_Q})}$ is attainable (up to log factors) by standard model selection techniques \citep[][Theorem 7]{koltchinskii:06}, 
we will aim for a transfer rate $\phi^\sharp(i)$, defined below, that incorporates this term at level $\Hyp_i$. 

We fix some $ \delta > 0$, and sequence of $\delta_i >0$ satisfying $\sum_i \delta_i \leq \delta$.
For instance, 
$\delta_i = \frac{1}{i(i+1)} \delta$. 




\begin{definition} Define the following quantity, for some $C_0$ independent of all model parameters:
    \begin{equation}
    \phi^\sharp (i) \defeq
    \min \left\{ 
        \E_Q(\hstar_{P, i}) +   C_0 \cdot C_{\rho_i} \left(\frac{d_i \; \log(n_P/\delta_i) }{n_P}\right)^{\frac{1}{(2 - \beta_{P, i}) \rho_{i}}}, 
        C_0 \left(\frac{d_{i^*_Q} \; \log(n_Q/\delta_{i^*_Q})}{n_Q}\right)^{\frac{1}{(2 - \beta_Q)}}
    \right\}. 
\end{equation}
\end{definition}

Since $C_{\rho_i}, \rho_i, \beta_{P, i}$ are not uniquely defined, without loss of generality we may take them to be the valid values which minimize $\phi^{\sharp}(i)$.
We have the following adaptive upper-bound.  

\begin{theorem}[Adaptive Upper-bound]\label{thm: upper bound}
There exists a proper learner $\hat{h}$, with no prior distributional knowledge beyond $\braces{d_i}$, which, with probability at least $1 - 3 \delta $, for a suitable value of $C_0$ achieves: 
\begin{equation}
     \E_Q(\hat{h}) \le \phi^\sharp (i^*_P).
\end{equation}
\end{theorem}

For sanity check, notice that if $P$ were equal to $Q$, then $\rho_{i^*_P} = 1$ is admissible and we recover the usual model selection bound in terms of $\max\braces{n_P, n_Q} \propto (n_P + n_Q)$. The bound is never worse than model selection under $Q$ alone, and can improve significantly for $P$'s \emph{close} to $Q$, i.e., with small $\rho_{i^*_P}, \E_Q(\hstar_P)$. 

As stated in the introduction, while SRM, a.k.a. \emph{complexity regularization} approaches are prevalent in the literature and in practice, it is unclear whether such approaches can adaptively achieve the above rate of $\phi^\sharp(i^*_P)$. Instead we employ an approach, similar to so-called \emph{Lepski's method}, based on intersections of empirical confidence balls (see Algorithm \ref{alg: upper bound}).  

We now turn to whether the rate $\phi^\sharp(i^*_P)$ is the best achievable. First, recalling the simple neural-nets Examples \ref{ex:NNthresholds} and \ref{ex:NN-Relu}, we remark that \emph{there exists situations, i.e., pairs of distributions $(P, Q)$ for which $i^\sharp \doteq \arg\min_i \phi^\sharp(i)$ is smaller than $i^\star_P$}. The simplest way to see this is to notice in these examples that we may have all $\E_Q(\hstar_{P, i})$ equal (or nearly equal) across levels, while at the same time $\rho_i$'s are non-decreasing in these examples, forcing a choice of $i^\sharp$ anywhere below $i^*_P$. This is illustrated in \Cref{fig:flatNNtransfer}, and formalized in \Cref{prop: i sharp } of Appendix \ref{app: i sharp example proof}.

The next result, relying on a second Algorithm \ref{alg: single level phi}, states that the better rate $\min_i \phi^\sharp(i)$ is indeed achievable \emph{given some distributional knowledge}.

\begin{theorem}[Oracle Upper-bound] \label{thm: oracle upper bound.}
There exists a proper learner $\hhat$ which, given knowledge of $\arg\min_i \phi^\sharp(i)$, guarantees with probability of at least $1 - 3\delta$,

\begin{equation}
    \E_Q(\hhat) \le \min_i \phi^\sharp(i).
\end{equation} 
\end{theorem}


Unfortunately, as we discuss in the next section, this oracle bound is not achievable adaptively. 

\begin{figure}\label{fig:flatNNtransfer}
\centering 
\includegraphics[width=8cm]{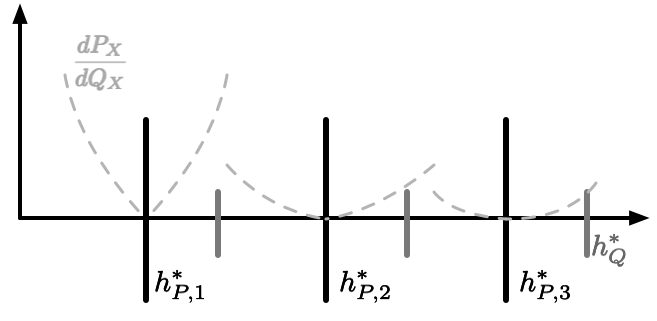}
\caption{\footnotesize A simple example, following up on NN Examples \ref{ex:NNthresholds} and \ref{ex:NN-Relu}, where $i^\sharp \doteq \arg\min_i \phi^\sharp(i)$ is different from $i^*_P$. Here, decision boundaries under $P$ are depicted in black, whereby each $\hstar_{P, i}$, $i = 1, 2, 3$, corresponds to the $i$ boundaries on the left of it, including those of $\hstar_{P, i-1}$ (level $\Hyp_i$ allows up to $i$ boundaries). Now decision boundaries under $Q$ (as depicted in gray) are shifted to the right of boundaries under $P$: as a consequence all $\hstar_{P_i}$'s have similar excess $Q$-error $\E_Q$, so that $i^\sharp$ is determined by $\rho_i$'s. Now for this hierarchy, $\rho_i$ may decrease (better transferability) for smaller levels $i$ simply by virtue of $P$ assigning more mass to corresponding decision boundaries as $i$ decreases, as suggested by the density $d{P_X}/dQ_X$ which is depicted in dashed lines.}
\end{figure}

\subsection{Adaptivity Gap}
The following quantity $\phi_\flat(i)$ is of similar order as $\phi^\sharp(i)$ up to log terms, provided $\log(1/\delta_i) \propto d_i$.  

\begin{definition} Define the following quantity: 
\begin{equation} 
    \phi_\flat(i) = \min \left\{  
    C_{\rho_i} \left(\frac{d_i}{n_P}\right)^{1/(2 - \beta_{P, i})\rho_i} + \E_Q(\hstar_{P, i}), \left(\frac{d_{i^*_Q}}{n_Q} \right)^{1/ (2 - \beta_Q)} \right\}. 
\end{equation}
\end{definition}

Our aim is to not only establish the un-achievability of the above oracle rate $\min_i \phi(i)$ by adaptive procedures, but also to try and pinpoint the sources of such hardness, i.e., decouple the effect of $\rho_i$'s and $E_Q(\hstar_{P, i})$'s. 
To this end, since these terms only pertain to transfer from $P$, we need only consider situations where the terms in $\phi(i)$ involving $i$, achieves the $\min$ in the definition of $\phi(i)$. 

Our first result below holds every parameter other than $\rho_i$'s fixed, and show that even then $\min_i \phi(i)$ cannot be achieved adaptively. In particular the construction sets $\E_Q(\hstar_{P, i}) = 0$ for all $i$ in the hierarchy, but confuses the learner by randomizing which level below $i^*_P$ admits larger $\rho_i$'s. Proofs of the next theorem is in \Cref{sec: lower bound proofs}.

{For simplicity, the construction in the next theorem sets all $\beta$'s to 1}. 
We give a similar result to the next theorem for a richer model class in \Cref{sec: adaptivity lower bounds for a larger model class}.

\begin{remark}[Irrelevance of Noise Regimes] 
    We remark that in other transfer learning settings with a single hypothesis class, including multi-task and multi-source learning, \emph{regimes of noise} as captured by $\beta$'s affect the extent to which adaptivity is possible; for instance in the low-noise regime $\beta_P = \beta_Q =1$ together with $\E_Q(h^*_P)= 0$ even na\"{i}ve pooling (where $P$ and $Q$ samples are combined into one) is adaptive, but non-adaptive for $\beta_P < 1$ \citep{hanneke2019value, hanneke2022no}. In the present setting of model selection however, the regime of noise in itself plays no role.
\end{remark}

\begin{theorem}[Oracle Rate is Not Achievable]\label{thm: simple rho lower bound }
There exists a hierarchy $\Hyp_1 \subset \Hyp_2 $, with $d_1, d_2 = 1$ satisfying the following. Pick any 
$\rho_a > \rho_b \geq 1$, and any $n_P$ and $n_Q$, where   $ \left( \frac{1}{32 n_P}\right)^{1/\rho_a} \le \frac{1}{32 n_Q}$. Then there is a family of distributions 
    $\{ \paren{P_\sigma, Q_\sigma}\}$ , indexed by some $\sigma$, such that the following hold. 
\begin{enumerate}[leftmargin=11pt]
        \item[(i)] $\forall \sigma $, transfer exponents from $P_\sigma$ to $Q_\sigma$
        are the set $\{\rho_1,\rho_2\} = \{\rho_a,\rho_b\}$ , with $C_{\rho_1}=C_{\rho_2} = 1$.

         \item[(ii)] $\forall \sigma$, we have 
        $\min_{i} \phi_\flat(i) = \left( \frac{1}{n_P}\right)^{1/\rho_b}$, strictly less than $\max_{i} \phi_\flat(i) = \left( \frac{1}{n_P}\right)^{1/\rho_a}$.
    
\end{enumerate} 
    \begin{align}
        \text{We have that, } \forall \hat h, \quad \sup_\sigma \Prob_{P_\sigma^\np \times Q_\sigma^\nq} \left[  \E_{Q_\sigma} (\hat{h}) \ge \frac{1}{256}\cdot \max_{i} \phi_\flat(i)\right] \ge 1/8 .
    \end{align}

\end{theorem}

The construction fixes $i^*_P = 2$, and randomizes which of $\rho_1\neq \rho_2$ takes the largest value in $\{\rho_a, \rho_b\}$. We note that our adaptive upper-bound $\max_\sigma \phi^\sharp(i^*_P)$ matches the lower bound $\max_i \phi_\flat (i)$ up to $\log$ terms. Also notice that, as $\rho_a, \rho_b$ are arbitrary, the lower bound can be  arbitrarily worse than the Oracle upper-bound, i.e., we can construct any gap in $[0, 1]$. 

The next class of distributions instead fixes $\rho_i$'s and allows $\E_Q(\hstar_{P, i})$ to vary. It builds on a similar intuition as for the proof of Theorem \ref{thm: basic lower bound}, and is included for completeness. 

\begin{theorem}\label{thm: lower bound for adaptivity to erisk across levels}
    Let $\Hyp_1 \subset \Hyp_2$ be a model class hierarchy such that there exists a set of two points that $\Hyp_2$ shatters but $\Hyp_1$ does not, and  assume that $\Hyp_1$ is non empty.
    Then for any $ 1\geq \alpha \ge 0$, $n_P$, and $n_Q$ such that $\frac{1}{ 2 n_Q} \ge \alpha $, there exists a class of distribution $\braces{(P_\sigma, Q_\sigma)}$ parameterized by $\sigma \in \{1, 2\}$, with $\beta_{P, \sigma} = \beta_{Q, \sigma} = 1$, where, for every $\sigma$, $ \alpha = \max \braces{ \E_{Q_\sigma}(\hstar_{P, 1}), \E_{Q_\sigma}(\hstar_{P, 2})}$, satisfying the following. For any classifier $\hat{h}$ that has access to $\np$ source and $\nq$ target samples,
    \begin{align}
        \sup_{\sigma \in \braces{1, 2}} \Prob_{P_\sigma^{\np} \times Q_\sigma^{\nq}} \brackets{\E_{Q_\sigma}(\hat{h}) \ge \alpha } \ge \frac{1}{4}.
    \end{align}
\end{theorem}

The proof of this theorem is given in \Cref{app sec: lower bound proofs}.

\section{Analysis}\label{sec:analysis}

\subsection{Proofs for Upper-bounds}

\begin{definition}[Empirical Minimal Sets] \label{def : empirical level set}
Let $\genGap{n_\dist}{\delta}{\classComp{\Hyp_i}} \defeq \frac{d_i \log (n_\dist/ d_i ) + \log 1/\delta}{n_\dist}$. 
Given $n_\dist$ samples from distribution $\dist $, define the empirical minimal set for hypothesis class $\Hyp_i$ to be
\begin{equation}
    \hHyp_i^\dist  \defeq \left \{
        h \in \Hyp_i \mid 
        \Hat{\Risk}_\dist(h) - \Hat{\Risk}_\dist (\hhat_{\dist, i}) 
            \le 
        C \left(\P_{n_\dist}[h \neq \hhat_{\dist, i}] \cdot  \genGap{n_\dist}{\delta_i}{\classComp{\Hyp_i}}\right)^{1/2} + c \genGap{n_\dist}{\delta_i}{\classComp{\Hyp_i}}
    \right \},
\end{equation}
where $\hhat_{\dist, i}$ denotes an ERM over $\Hyp_i$ computed using samples from distribution $\dist$.
\end{definition}

We assume that in addition to the source and target training sets, we are also given a hold-out target sample set $S'_Q$ of size $n_Q$. Let $\emRisk'_Q(\cdot)$ denote the empirical risk and $\P'_\nq$ denote the empirical distribution on these held out samples.

The main algorithm is presented next, and relies on \Cref{alg: single level phi}. 

\begin{footnotesize}
\begin{algorithm}[H]
    \caption{Adaptive Trade-off}
    \label{alg: upper bound}
    \begin{algorithmic}
        \State \textbf{Input:} $S_P, S_Q$,  $S'_Q$
        \State Compute
        $
            \hat{i}_P = \min i \; \; \;
            \mbox{s.t.}\;  \bigcap_{j\ge i}^\infty \hHyp_j^P \neq \emptyset.
        $
        \State Compute $ \tilde{\Hyp}^P = \bigcap_{j\ge \hat{i}_P}^\infty \hHyp_j^P$.
        \State Return output of \Cref{alg: single level phi} with $S_P, S_Q, S'_Q$ and the set $\tilde{\Hyp}^P$
    \end{algorithmic}
\end{algorithm}

\begin{algorithm}[H]
    \caption{ Tradeoff on $Q$, at level $\Hyp_i$}
    \label{alg: single level phi}
    \begin{algorithmic}
        \State \textbf{Require:} Any subset $\tilde{\Hyp}^P \subseteq \hHyp_i^P$
        \State \textbf{Input:} $ S_Q, S'_Q, \tilde{\Hyp}^P $
        \State  Compute
        $ 
            \hat{i}_Q = \min i \;\;\;
            \mbox{s.t.}\;  \bigcap_{j \ge  i}^\infty \hHyp_j^Q \neq \emptyset.
        $
        \State Pick  $
             \hat{h}_Q  \in   \bigcap_{j\ge \hat{i}_Q}^\infty \hHyp_j^Q
        $
        and pick $\hat{h}_{P,i} \in \tilde{\Hyp}^P$.

        \State\textbf{If}  $\emRisk_Q'(\hhat_{P, i}) - \emRisk_Q'(\hat{h}_Q) \le  \paren{\P'_{n_Q}[\hhat_{P, i} \neq \hat{h}_Q] \cdot \genGap{n_Q}{\delta}{1}}^{1/2} + c \genGap{n_Q}{\delta}{1}$:
        \State \textbf{then}
        $\hat{h}_i \leftarrow \hat{h}_{P, i}$
         \; \textbf{else:} {$\;\hat{h}_i \leftarrow \hat{h}_Q$}
    \end{algorithmic}
\end{algorithm}
\end{footnotesize}

The next lemma gives guarantees for \Cref{alg: single level phi}, and is essential to our main upper-bounds.

\begin{lemma}\label{lem: single level bounds.}
    Let $\hhat_i$ be the output of \Cref{alg: single level phi}. With probability of at least $1 - 3 \delta$ over the samples $S_Q, S'_Q$ and $S_P$ (which is used to construct $\hHyp^P_i$) 
    \begin{equation}
        \E_Q(\hat{h}) \le \phi^\sharp(i).
    \end{equation}
    
\end{lemma}
The proof is given in  \Cref{app sec: remaining proofs for upper bounds}. The proofs of main upper-bound results are given next.

\begin{proofof} {\Cref{thm: upper bound}}
Let $\hat{h}$ be the output of \Cref{alg: upper bound}. 
Note that under the same events where the bound in \Cref{lem: single level bounds.} holds, using the same arguments as in \Cref{claim: upper bound on i hat Q} we can conclude that with probability of at least $1- \delta$, $\hat{i}_P \le i^*_P$.  Consequently  $\tilde{\Hyp}^P \subseteq \hHyp_{i^*_P}^P$ . Since \Cref{alg: upper bound} returns the output of \Cref{alg: single level phi} on a subset of $\tilde{\Hyp}^P$, it enjoys the guarantees as \Cref{alg: single level phi} for level $i^*_P$. Therefore, the bound in \Cref{lem: single level bounds.} applies to the output of \Cref{alg: upper bound} at $i = i^*_P$.
\end{proofof}

The proof of the oracle upper is also a simple application of \Cref{lem: single level bounds.}.

\begin{proofof}{ \Cref{thm: oracle upper bound.}}
    Let $i^\sharp \defeq \argmin_i \phi^\sharp(i)$.
    Given $i^\sharp$, oracle would then run \Cref{alg: single level phi} with given samples and $\hHyp_{i^\sharp}^P$ as input. 
    Applying \Cref{lem: single level bounds.} to the output would prove the statement of the theorem.
\end{proofof}

\subsection{Proofs for Lower-bounds}\label{sec: lower bound proofs}

\subsubsection{Proof of \Cref{thm: simple rho lower bound }}
We start with a construction, defining a suitable hierarchy and distributions. 

\paragraph{Construction.} 
Let $\X = [0, 1]$.
 We let $\Hyp_1$ contain only two one sided threshold classifiers, and $\Hyp_2$ contains $\Hyp_1$ plus two one sided interval classifiers. Let $r = 1/9$. The one sided threshold classifiers in $\Hyp_1$ are $h_1(x) = \sgn{x - 2/3}$ and $h'_1(x) = \sgn{x - 2/3 + r}$. The one sided intervals $h_2$ and  $h'_2$ only positively label the set of points in $[1/9, 1/3]$ and $[1/9, 1/3+ r]$ respectively.

We construct a family of four  distributions $ \{ (P_\sigma, Q_\sigma)\}_{\sigma \in \{ \pm 1\}^2}$, where each $P_\sigma$ and $Q_\sigma$ is  supported over $[1/9, 1] \times  \{\pm 1\} $.  Throughout this section we drop the subscript $\sigma$  when a quantity is the same for all distributions in the family.
We refer to the intervals $[1/9, 1/3], [1/3, 1/3 + r], [2/3- r, 2/3]$ and $[2/3, 1]$ as $\Lo, \Li, \Ri$ and $\Ro$ respectively.

For the marginals, we assume that within each interval the mass is uniformly distributed.
 Let $ P_X$ and $Q_{X, \sigma}$, be the marginal distributions under source and target respectively. 
All the distributions in the family have the same source marginal distribution $P_X$, which has $P_X(\Lo) = 1/3, P_X(\Li) = P_X(\Ri) = \frac{1}{c_1 n_P}$, $P_X([1/3 + r, 2/3 - r]) = \frac{5}{12} - \frac{2}{c_1 n_P}$, and $P_X(\Ro) = \frac{1}{4}$.
The constant $c_1$ is set to $32$, the reason for which becomes clear in \Cref{claim: lower bound typical samples} . The labels for the source are $Y_P(\Lo) = Y_P(\Ro) = +1$, and for the rest of intervals the labels are the same as $Q$ for $\sigma$.

The target marginal distribution $Q_{X, \sigma} (\Li)$ and $Q_{X, \sigma} (\Ri)$ depends on $\sigma_2$. If 
 $\sigma_2 = +1$, set $Q_{X, (\sigma_1, + 1)} (\Li) = \left(\frac{1}{c_1 n_P}\right)^{1/\rho_a}$ and 
 $Q_{X, (\sigma_1, + 1)} (\Ri) = \left(\frac{1}{c_1 n_P}\right)^{1/\rho_b}$, while if $\sigma_2 = -1$, 
  $Q_{X, (\sigma_1, -1)} (\Li) = \left(\frac{1}{c_1 n_P}\right)^{1/\rho_b}$ and 
 $Q_{X, (\sigma_1, -1 )} (\Ri) = \left(\frac{1}{c_1 n_P}\right)^{1/\rho_a}$.
 
Let $\Delta \defeq \left(\frac{1}{c_1 n_P}\right)^{1/\rho_a} - \left(\frac{1}{c_1 n_P}\right)^{1/\rho_b}$,  for the rest of the intervals, $Q_X(\Lo) = Q_X(\Ro) = \frac{1}{2} \Delta$, and finally $Q_X([1/3 + r, 2/3 - r]) = 1 - 2 \left(\frac{1}{c_1 n_P}\right )^{1/\rho_a}$.

For all $\sigma$, labels are noiseless. Let $Y_{\sigma}(A)$ denote the label of the set $A$ under $\sigma$.
We set $Y_{Q,\sigma} (\Lo) = -\sigma_1 \sigma_2$, $Y_{Q, \sigma}(\Ro) = \sigma_2$, $Y_\sigma(\Li) = Y_\sigma(\Ri) = \sigma_1$ and $Y([1/3 + r, 2/3 - r]) = -1$.

We make the following two claims, which imply statements (i) and (ii) of the theorem. Additionally, $\beta_{P, \sigma} = \beta_{Q, \sigma} = 1$, since the labels are noiseless.

\begin{claim}
     For every $\sigma$ and $i \in \{1, 2\}$, $\E_{Q_\sigma}(\hstar_{P_\sigma, i}) = 0$.
\end{claim}
\begin{proof}
For every $\sigma$ under $Q_\sigma$ there are two risk minimizers, one in $\Hyp_1$ and another in $\Hyp_2 \setminus \Hyp_1$. Specifically, when $\sigma_1 = +1$, both $h'_2$ and $h'_1$ are risk minimizers, with risk $\left(\frac{1}{c_1 n_P}\right )^{1/\rho_a}$, since the one that mislabels the inner interval with mass $\left(\frac{1}{c_1 n_P}\right )^{1/\rho_b}$ will also mislabel $\Lo \cup \Ro$.

    On the other hand, when $\sigma_1 = -1$, since the regions $\Lo$ and $\Ro$ have the same sign, each of $h_1$ and $h_2$ will mislabel exactly one of them, which results in the minimum risk of $\frac{\Delta}{2}$.
    It is easy to see that both of $h'_1$ and $h'_2$ have a strictly larger risk. 
\end{proof}

\begin{claim} The following holds for every value of $\sigma_1$. 
If $\sigma_2 =1$, we have $\rho_1 = \rho_b $ and $\rho_2 = \rho_a$. 
Otherwise, for $\sigma = -1$, we have 
$\rho_1 = \rho_a $ and $\rho_2 = \rho_b$.
Furthermore, for all $\sigma$, $C_{\rho_1} = C_{\rho_2} = 1$.
\end{claim}
\begin{proof}
    First consider $\Hyp_1$.
    Suppose that $\sigma_2 = 1$, since $Q_{X, \sigma}(\Ri)= \left(\frac{1}{c_1 n_P}\right)^{1/\rho_b}$, whichever of $h_1$ or $h'_1$ that is not a risk minimizer under source and target, will have excess risk of $\frac{1}{c_1 n_P}$ under source and  $\left(\frac{1}{c_1 n_P}\right)^{1/\rho_b}$ under target, which means that $\rho_b$ is a transfer exponent with coefficient one.
    When $\sigma_2 = -1$, since the region where $h_1$ and $h'_1$ differ has mass $\frac{1}{c_1n_P}$ under source and $\left(\frac{1}{c_1 n_P}\right)^{1/\rho_a}$ under target,  $\rho_a$ is a transfer exponent with respect to $\Hyp_1$ with coefficient one. 
    
    For $\Hyp_2$, note that every $h \in \Hyp_1$ has an excess risk of at least $1/3 - 1/4 = 1/12 >  \frac{1}{c_1 n_P} = \frac{1}{32 n_P}$ under source and excess risk of at most  $\left( \frac{1}{c_1 n_P}\right)^{1/\rho_a}$ under target, so the transfer exponent condition  with $\rho_b$ or $\rho_a$  and coefficient one holds trivially.
    For hypotheses that are in $\Hyp_2 \setminus \Hyp_1$, since one of them is a risk minimizer, and the region they differ has mass $\left(\frac{1}{c_1 n_P}\right)^{1/\rho_a}$ or $\left(\frac{1}{c_1 n_P}\right)^{1/\rho_b}$  under target and $\frac{1}{c_1n_P}$ under source, then depending on $\sigma_2$, either   $\rho_b$ or $\rho_a$ would be a transfer exponent with coefficient one.  
\end{proof}

The next proposition shows that for every possibly improper learner, there is a distribution in the family under which the learner has high excess risk. The proof is given in the appendix.

\begin{proposition}\label{prop: adaptivitiy bad distributions}
Let $c_1 = 32$ in the construction.
For any classifier $\tilde{h}$, possibly improper, there exists $\sigma \in \{\pm 1\}^2$ such that $\E_{Q_{\sigma}} (\tilde{h}) \ge \lv $.
\end{proposition}

Let $\Pi_{\sigma} \defeq P^{\np} \times Q^\nq$ and $S_\sigma \sim \Pi_\sigma$ be the source and target samples. 
The next claim defines the event $B$ and lower bounds its probability.

\begin{claim} \label{claim: lower bound typical samples}
Let $B$ be the event that of all $n_P$ source and $n_Q$ target samples fall in the intervals $\Lo \cup [1/3 + r, 2/3 - r] \cup \Ro$ under source  and $[1/3 + r, 2/3 - r] $ under target. Then we may choose $c_1$ (from the definition of marginal distributions) such that for all $\sigma \in \{\pm 1\}^2$, $\Pi_\sigma [B] \ge  7/8$.

\end{claim}
\begin{proof}
    For any $\sigma$, 
    \begin{align}
        \Pi_\sigma [B] &= \left( 1 - \frac{2}{c_1 n_P} \right)^{n_P}
        \left(1 - 2 \left(\frac{1}{c_1 n_P}\right)^{1/\rho_a} \right)^{n_Q}
        \ge 
        \left(1 - \frac{2 n_P}{c_1 n_P}\right)
        \left(1 - 2 n_Q (\frac{1}{c_1 n_P})^{1/\rho_a}\right),
    \end{align}
where the inequality follows by Bernoulli's inequality.  By the assumption that $(\frac{1}{c_1 n_P})^{1/\rho_a} \le \frac{1}{32 n_Q}$ and picking $c_1 = 32$, we can  ensure $\Pi_\sigma [B] \ge 7/8$.
\end{proof}

 \begin{proofof}{\Cref{thm: simple rho lower bound }}
     
Let $\hat{h}$ be a classifier that is output by a learning algorithm that has access to samples $S_\sigma$.
The lower bound follows by randomizing the choice of $\sigma$. Suppose that $\hat{\sigma}$ is sampled uniformly at random from $\{\pm 1\}^2$, then

\begin{align}
    \sup_{\sigma} \Prob_{\Pi_\sigma}\brackets{\E_{Q_\sigma}(\hat{h}) \ge \lv}
    \ge &
    \expecf{\hat{\sigma}} \exf{S_{\hat{\sigma}}}{ 
        \indic{\E_{Q_\sigma}(\hat{h}) \ge \lv }}
    \\ = &
    \expecf{S_{\hat{\sigma}}} \exf{\hat{\sigma} \mid S_{\hat{\sigma}}}{ 
        \indic{\E_{Q_\sigma}(\hat{h}) \ge \lv }}
    \\ \ge & 
     \expecf{S_{\hat{\sigma}}} \exf{\hat{\sigma} \mid S_{\hat{\sigma}}}{ 
        \indic{\E_{Q_\sigma}(\hat{h}) \ge \lv } \cdot \indic{B}}.
\end{align}

By construction, $\Prob_{\hat{\sigma} \mid S_{\hat{\sigma}}, B}(\sigma) = \Prob_{\hat{\sigma}}(\sigma) = 1/4$. Let $\tilde{\sigma}$ index the distribution that results in high $\hat{h}$ excess risk as in \Cref{prop: adaptivitiy bad distributions}. We have 

\begin{align}
     \expecf{S_{\hat{\sigma}}} \exf{\hat{\sigma} \mid S_{\hat{\sigma}}}{ 
        \indic{\E_{Q_\sigma}(\hat{h}) \ge \lv } \cdot \indic{B}}
    &\ge  \expecf{S_{\hat{\sigma}}} \exf{\hat{\sigma} \mid S_{\hat{\sigma}}}{ 
        \indic{\hat{\sigma}= \tilde{\sigma}} \cdot \indic{B}}
    \\ & = 
    \expecf{S_{\hat{\sigma}}} \exf{\hat{\sigma} \mid S_{\hat{\sigma}}}{ 
        \indic{\hat{\sigma}= \tilde{\sigma}} \cdot \indic{B}}
    \\ & = 
    \frac{1}{4} \cdot \Prob_{\Pi_{\hat{\sigma}}} \brackets{B} 
    \ge \frac{7}{32}.
\end{align}
\end{proofof}

\section*{Conclusion} 
 We have shown that source data can help significantly improve target risk under model selection; however, adaptive rates do not always match oracle rates in the model selection setting, as we exhibit situations where no procedure can attain oracle rates without distributional knowledge. Even more striking is that the gap between optimal adaptive rates and oracle minimax rates can be arbitrary, which is not often the case in minimax theory. However this leaves open the possibility of smaller or more controlled gaps under, e.g., further structural assumptions on the model hierarchy.

\section*{Acknowledgments}
We thank COLT reviewers and AC for useful comments that help improve the manuscript

\bibliography{refs}

\begin{thebibliography}{32}
\providecommand{\natexlab}[1]{#1}
\providecommand{\url}[1]{\texttt{#1}}
\expandafter\ifx\csname urlstyle\endcsname\relax
  \providecommand{\doi}[1]{doi: #1}\else
  \providecommand{\doi}{doi: \begingroup \urlstyle{rm}\Url}\fi

\bibitem[Achille et~al.(2019)Achille, Paolini, Mbeng, and
  Soatto]{achille2019information}
Alessandro Achille, Giovanni Paolini, Glen Mbeng, and Stefano Soatto.
\newblock The information complexity of learning tasks, their structure and
  their distance.
\newblock \emph{arXiv:1904.03292}, 2019.

\bibitem[Aliprantis et~al.(2006)Aliprantis, Harris, and Tourky]{Aliprantis2006}
Charalambos~D. Aliprantis, David Harris, and Rabee Tourky.
\newblock Continuous piecewise linear functions.
\newblock \emph{Macroeconomic Dynamics}, 10\penalty0 (1):\penalty0 77--99, 02
  2006.

\bibitem[Ando and Zhang(2005)]{ando2005framework}
Rie~Kubota Ando and Tong Zhang.
\newblock A framework for learning predictive structures from multiple tasks
  and unlabeled data.
\newblock \emph{Journal of Machine Learning Research}, 6\penalty0
  (Nov):\penalty0 1817--1853, 2005.

\bibitem[Arora et~al.(2019)Arora, Khandeparkar, Khodak, Plevrakis, and
  Saunshi]{arora2019theoretical}
Sanjeev Arora, Hrishikesh Khandeparkar, Mikhail Khodak, Orestis Plevrakis, and
  Nikunj Saunshi.
\newblock A theoretical analysis of contrastive unsupervised representation
  learning.
\newblock \emph{arXiv:1902.09229}, 2019.

\bibitem[Bartlett et~al.(2006)Bartlett, Jordan, and Mc{A}uliffe]{bartlett:06b}
P.~Bartlett, M.~I. Jordan, and J.~Mc{A}uliffe.
\newblock Convexity, classification, and risk bounds.
\newblock \emph{Journal of the American Statistical Association}, 101\penalty0
  (473):\penalty0 138--156, 2006.

\bibitem[Ben-David et~al.(2007)Ben-David, Blitzer, Crammer, and
  Pereira]{ben2007analysis}
Shai Ben-David, John Blitzer, Koby Crammer, and Fernando Pereira.
\newblock Analysis of representations for domain adaptation.
\newblock In \emph{Advances in Neural Information Processing Systems}, 2007.

\bibitem[Ben-David et~al.(2010)Ben-David, Blitzer, Crammer, Kulesza, Pereira,
  and Vaughan]{ben2010theory}
Shai Ben-David, John Blitzer, Koby Crammer, Alex Kulesza, Fernando Pereira, and
  Jennifer~Wortman Vaughan.
\newblock A theory of learning from different domains.
\newblock \emph{Machine Learning}, 79\penalty0 (1-2):\penalty0 151--175, 2010.

\bibitem[Cortes et~al.(2008)Cortes, Mohri, Riley, and
  Rostamizadeh]{cortes2008sample}
Corinna Cortes, Mehryar Mohri, Michael Riley, and Afshin Rostamizadeh.
\newblock Sample selection bias correction theory.
\newblock In \emph{International Conference on Algorithmic Learning Theory},
  2008.

\bibitem[Crammer et~al.(2008)Crammer, Kearns, and Wortman]{crammer2008learning}
Koby Crammer, Michael Kearns, and Jennifer Wortman.
\newblock Learning from multiple sources.
\newblock \emph{Journal of Machine Learning Research}, 9\penalty0
  (Aug):\penalty0 1757--1774, 2008.

\bibitem[Du et~al.(2020)Du, Hu, Kakade, Lee, and Lei]{du2020few}
Simon~S Du, Wei Hu, Sham~M Kakade, Jason~D Lee, and Qi~Lei.
\newblock Few-shot learning via learning the representation, provably.
\newblock \emph{arXiv:2002.09434}, 2020.

\bibitem[Gretton et~al.(2009)Gretton, Smola, Huang, Schmittfull, Borgwardt, and
  Sch{\"o}lkopf]{gretton2009covariate}
Arthur Gretton, Alex Smola, Jiayuan Huang, Marcel Schmittfull, Karsten
  Borgwardt, and Bernhard Sch{\"o}lkopf.
\newblock Covariate shift by kernel mean matching.
\newblock In \emph{Dataset Shift in Machine Learning}, pages 131--160, 2009.

\bibitem[Hanneke and Kpotufe(2019)]{hanneke2019value}
Steve Hanneke and Samory Kpotufe.
\newblock On the value of target data in transfer learning.
\newblock In \emph{Advances in Neural Information Processing Systems}, 2019.

\bibitem[Hanneke and Kpotufe(2022)]{hanneke2022no}
Steve Hanneke and Samory Kpotufe.
\newblock A no-free-lunch theorem for multitask learning.
\newblock \emph{The Annals of Statistics}, 50\penalty0 (6):\penalty0
  3119--3143, 2022.

\bibitem[Jalali et~al.(2010)Jalali, Sanghavi, Ruan, and
  Ravikumar]{jalali2010dirty}
Ali Jalali, Sujay Sanghavi, Chao Ruan, and Pradeep Ravikumar.
\newblock A dirty model for multi-task learning.
\newblock In \emph{Advances in Neural Information Processing Systems}, 2010.

\bibitem[Koltchinskii(2006)]{koltchinskii:06}
V.~Koltchinskii.
\newblock Local {R}ademacher complexities and oracle inequalities in risk
  minimization.
\newblock \emph{The Annals of Statistics}, 34\penalty0 (6):\penalty0
  2593--2656, 2006.

\bibitem[Kpotufe and Martinet(2018)]{kpotufe2018marginal}
Samory Kpotufe and Guillaume Martinet.
\newblock Marginal singularity, and the benefits of labels in covariate-shift.
\newblock \emph{arXiv:1803.01833}, 2018.

\bibitem[Lounici et~al.(2011)Lounici, Pontil, Van De~Geer, Tsybakov,
  et~al.]{lounici2011oracle}
Karim Lounici, Massimiliano Pontil, Sara Van De~Geer, Alexandre~B Tsybakov,
  et~al.
\newblock Oracle inequalities and optimal inference under group sparsity.
\newblock \emph{The Annals of Statistics}, 39\penalty0 (4):\penalty0
  2164--2204, 2011.

\bibitem[Mansour et~al.(2009{\natexlab{a}})Mansour, Mohri, and
  Rostamizadeh]{mansour2009domain}
Yishay Mansour, Mehryar Mohri, and Afshin Rostamizadeh.
\newblock Domain adaptation: Learning bounds and algorithms.
\newblock \emph{arXiv:0902.3430}, 2009{\natexlab{a}}.

\bibitem[Mansour et~al.(2009{\natexlab{b}})Mansour, Mohri, and
  Rostamizadeh]{mansour2009multiple}
Yishay Mansour, Mehryar Mohri, and Afshin Rostamizadeh.
\newblock Multiple source adaptation and the {R}{\'e}nyi divergence.
\newblock In \emph{Proceedings of the 25th Conference on Uncertainty in
  Artificial Intelligence}, 2009{\natexlab{b}}.

\bibitem[Massart and N\'{e}d\'{e}lec(2006)]{MN:06}
P.~Massart and \'{E}. N\'{e}d\'{e}lec.
\newblock Risk bounds for statistical learning.
\newblock \emph{The Annals of Statistics}, 34\penalty0 (5):\penalty0
  2326--2366, 2006.

\bibitem[Maurer et~al.(2013)Maurer, Pontil, and
  Romera-Paredes]{maurer2013sparse}
Andreas Maurer, Massi Pontil, and Bernardino Romera-Paredes.
\newblock Sparse coding for multitask and transfer learning.
\newblock In \emph{International Conference on Machine Learning}, 2013.

\bibitem[Maurer et~al.(2016)Maurer, Pontil, and
  Romera-Paredes]{maurer2016benefit}
Andreas Maurer, Massimiliano Pontil, and Bernardino Romera-Paredes.
\newblock The benefit of multitask representation learning.
\newblock \emph{The Journal of Machine Learning Research}, 17\penalty0
  (1):\penalty0 2853--2884, 2016.

\bibitem[McNamara and Balcan(2017)]{mcnamara2017risk}
Daniel McNamara and Maria-Florina Balcan.
\newblock Risk bounds for transferring representations with and without
  fine-tuning.
\newblock In \emph{International Conference on Machine Learning}, 2017.

\bibitem[Mousavi~Kalan et~al.(2020)Mousavi~Kalan, Fabian, Avestimehr, and
  Soltanolkotabi]{mousavi2020minimax}
Mohammadreza Mousavi~Kalan, Zalan Fabian, Salman Avestimehr, and Mahdi
  Soltanolkotabi.
\newblock Minimax lower bounds for transfer learning with linear and one-hidden
  layer neural networks.
\newblock \emph{Advances in Neural Information Processing Systems},
  33:\penalty0 1959--1969, 2020.

\bibitem[Muandet et~al.(2013)Muandet, Balduzzi, and
  Sch{\"o}lkopf]{muandet2013domain}
Krikamol Muandet, David Balduzzi, and Bernhard Sch{\"o}lkopf.
\newblock Domain generalization via invariant feature representation.
\newblock In \emph{International Conference on Machine Learning}, 2013.

\bibitem[{Negahban} and {Wainwright}(2011)]{negahban5773043}
S.~N. {Negahban} and M.~J. {Wainwright}.
\newblock Simultaneous support recovery in high dimensions: Benefits and perils
  of block $\ell _{1}/\ell _{\infty} $-regularization.
\newblock \emph{IEEE Transactions on Information Theory}, 57\penalty0
  (6):\penalty0 3841--3863, 2011.

\bibitem[Pentina and Lampert(2014)]{pentina2014pac}
Anastasia Pentina and Christoph Lampert.
\newblock A {PAC}-{B}ayesian bound for lifelong learning.
\newblock In \emph{International Conference on Machine Learning}, 2014.

\bibitem[Tripuraneni et~al.(2020)Tripuraneni, Jordan, and
  Jin]{tripuraneni2020theory}
Nilesh Tripuraneni, Michael~I Jordan, and Chi Jin.
\newblock On the theory of transfer learning: {T}he importance of task
  diversity.
\newblock \emph{arXiv:2006.11650}, 2020.

\bibitem[Tsybakov(2004)]{tsybakov:04}
A.~B. Tsybakov.
\newblock Optimal aggregation of classifiers in statistical learning.
\newblock \emph{The Annals of Statistics}, 32\penalty0 (1):\penalty0 135--166,
  2004.

\bibitem[Tsybakov(2009)]{tsybakov2009introduction}
Alexandre~B Tsybakov.
\newblock \emph{Introduction to Nonparametric Estimation}.
\newblock Springer, 2009.

\bibitem[Vapnik and Chervonenkis(1971)]{VC:72}
V.~Vapnik and A.~Chervonenkis.
\newblock On the uniform convergence of relative frequencies of events to their
  expectation.
\newblock \emph{Theory of Probability and its Applications}, 16:\penalty0
  264--280, 1971.

\bibitem[Yang et~al.(2013)Yang, Hanneke, and Carbonell]{yang2013theory}
Liu Yang, Steve Hanneke, and Jaime Carbonell.
\newblock A theory of transfer learning with applications to active learning.
\newblock \emph{Machine learning}, 90\penalty0 (2):\penalty0 161--189, 2013.

\end{thebibliography}

\appendix
\section{Proofs of Propositions for Examples} \label{sec: main example proofs}

\subsection{Proof of \Cref{prop: threshold neural net}}

\begin{proofof}{ \Cref{prop: threshold neural net}}

We construct $P$ and $Q$ such that
marginal distributions $Q_X$ and $P_X$ are supported on $[0, 1]$. Let $Q_X$ be the uniform distribution over $[0, 1]$.  To define the source marginal distribution, we pick $L$  points $V = \{v_k\}_{k=1}^L$ on the unit interval so that each 
$v_k =  \frac{k}{L+1}$. Then define
\begin{equation} \label{eq: source one dim NN construction}
     \den_P(x) \propto \rho_m \cdot 2^{-2 m \cdot \rho_m} \abs{x - v_m}^{\rho_{m}  -1},
\end{equation}
where $v_m \in V$ is the closest point to $x$, and ties are broken by picking the smaller one, except when $x$ is in the first interval, in which case we set $m =1$. This leads to $L$ partitions $R_1, \dots R_L$ of the interval $[0, 1]$ such that for every $x \in R_i$, $\den_P(x) = \frac{\rho_i \cdot 2^{- (2 i \cdot \rho_i + i)}}{Z} \abs{x - v_i}^{\rho_i -1}$, where $Z$ is a normalizing constant. See \Cref{fig: denisity of P} for an example with $L= 3$.

\begin{figure}[H]
    \centering
    \begin{tikzpicture}
    \draw[thick,->] (0,0) -- (6,0);
    \draw[thick,->] (0,0) -- (0,4) node[anchor= south east] {$f_P$};
    
    \draw (0, 3) to [out=-80, in=150] (1.5, 0);
    \draw (1.5, 0) to [out = 30, in= -120]  (2.25, 0.75);

    \draw (2.25, 0.45) [out= -60, in=165 ] to (3, 0);
    \draw (3, 0) to [out=15, in =-130](3.75, 0.45);

    \draw (3.75, 0.23) to [out=-40, in= 170] (4.5, 0);
    \draw(4.5, 0) to  [out=10, in= -140] (6, 0.4);

    \foreach \x in {1.5,3,4.5}
        \draw [black!20!white, line width= 3] (\x, 0.5) -- (\x , -0.5);

    \foreach \x in {0.75, 3.75}
        \node at (\x, -0.25) {$+$};

    \foreach \x in {2.25, 5.25}
        \node at (\x, -0.25) {$-$};
    
    \draw [decorate,
    decoration = {calligraphic brace, mirror}, line width=1] (0,- 0.6) --  (2.25,-0.6);
    \node at   (1.125, -1) {$R_1$};
\end{tikzpicture}
    \caption{Construction of the marginal density of $P$ for the threshold neural neural net example, with $L=3$.}
    \label{fig: denisity of P}
\end{figure}
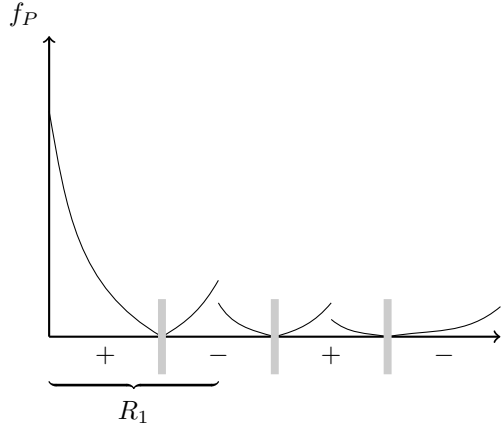

The labels for both source and target are given by $f_{\theta^*}$, where $\theta^* = (L, a^*, r^*, w^*, b^*)$ is a set of parameters for a risk minimizer. We pick these parameters such that $i^*_P = i^*_Q = L$ as follows.
First, set $w^*$ to the all ones vector and $b_i^* = v_i$ for all $i \le L$. The boundaries $b_i^*$ divide $[0, 1]$ into $L + 1$ intervals, distinct from the $L$ regions $R_1, \cdots R_L$. 
Next claim shows that we can pick $ a^*, r^*$ such that the label for these intervals are alternating, so that every point $b^*_i = v_i$ is indeed a decision boundary. Note that this leads to $\beta_{P, i} = \beta_{Q, i} = 1$. Proofs for all the claims in this proof appear at the end of this section.

\begin{claim}\label{prop: realizing alternating signs}
 Let $ 0 < b'_1 < \dots < b'_i < 1$ be an increasing sequence of points in $[0, 1]$ that partition the unit interval into $i + 1$ intervals $I_1 = [0, b'_1], I_2 = (b'_1, b'_2],  \dots,$ and $ I_{i+1} = (b'_i, 1]$. For any sign pattern $\sigma \in \{ \pm 1\}^{i + 1}$, there exists a set of parameters $\theta$, such that $f_\theta$
maps any $x \in I_j$ to $\sigma_j$ for all $ j \in [i+1]$. Furthermore, any two layer threshold neural net of the form \cref{eq: def two layer threshold}, that is, any $\theta$,  with $i$ hidden units can lead to at most $i$ decision boundaries.
\end{claim}

Next claim shows that there is risk minimizer in  $\Hyp_i$ that has the same decision boundaries as the smallest $i$ decision boundaries in $\theta^*$ and  correctly labels the first $i + 1$ intervals, by matching their signs with the signs of the first $i + 1$ intervals generated by $\theta^*$.

\begin{claim} \label{claim: construction of risk minimizers for lower levels}
 For any $ i \le L$, let $\theta^*_{i} = ( i, a^{*, i}, r^{*,i}, w^{*,i}, b^{*, i})$, where for every $j \in [i]$, $b^{*,i}_ j = b_j^*$, and $w^{*,i}$ is the all ones vector. The first $i$ intervals generated by $b^{*,i}$ are the same as those of $\theta^*$, and  by \Cref{prop: realizing alternating signs} we can pick $a^{*, i}$ and $r^{*,i}$ such that $\theta^*_{i}$ makes no error in the first $ i + 1$ intervals generated by $\theta^*$. Then 
 $f_{\theta^*_{i}}$ is a unique risk minimizer over the class $\Hyp_{i}$ under $P$.
\end{claim}

In the next proposition we show that for any $i \le L$, $\rho_i$ is a transfer exponent from $P$ to $Q$ with respect to $\Hyp_i$. Intuitively, we show that whenever there is error, it is dominated by the error in the regions determined by the first $i$ thresholds.

\begin{claim} \label{claim: te for TNN}
 For every $ 1 \le i \le L$, there exists a constant $ 0 < C_{\rho_i} < \infty$ such that 
 $\rho_i$ is a transfer exponent from $P$ to $Q$ with respect to $\Hyp_i$ with coefficient  $C_{\rho_i} \le (L+1) 2^{3i +1}$.
\end{claim}

Next, we show that for every $1 \le i \le L$ a transfer exponent from $P$ to $Q$ with respect to $\Hyp_i$ is lower bounded by $\rho_i$.
Fix a level  $ 1 \le i \le L$, and consider a sequence of classifiers $f_{\theta(t)}$ constructed so that $\theta(t)$ matches $\theta^*_{i} $ everywhere except for the last decision boundary $b^{*, i}_{i}$. That is, $\theta(t) = (i, a^*_{i}, r^{*, i}, w^{*, i},  b^t)$, where $b^t_j = b^{*,i}_ j$
for all $ j < i$, and for $ t <  \frac{1}{2(L+1)}$,
\begin{equation}
    b^t_i = b^{*,i}_i + t .
\end{equation}

It is easy to see that the interval $[ b^{*,i}_i, b^t_i]$ is the only disagreement region between $\theta^*_{i}$  and $ \theta(t)$ and has length $t$. By the construction of $P_X$ given in \cref{eq: source one dim NN construction}, and integrating over this region, the excess risk is

\begin{equation}
    \E_P(f_{\theta(t)}, f_{\theta^*_{i}}) = C t^{\rho_{i}},
\end{equation}
for some constant $C$. Since $Q_X$ is the uniform distribution,
\begin{equation}
    \E_Q (f_{\theta(t)}, f_{\theta^*_{i}}) = t.
\end{equation}

Now we argue that the minimal transfer exponent for this level is
at least $\rho_{i}$.
Suppose for contradiction, that there exists a transfer exponent  $\tilde{\rho}_{i} < \rho_{i} $. That would imply that there exists a constant $C_{\tilde{\rho}_{i}}$ such that for every $f_\theta \in \Hyp_{i}$, 

\begin{equation}
    \frac{
        \E_Q(f_\theta, f_{\theta^*_{i}})
    }{
        \E_P(f_\theta, f_{\theta^*_{i}})^{1/\tilde{\rho}_{i}}
        }
    \le C_{\tilde{\rho}_{i}}.
\end{equation}

However, for the sequence of $f_{\theta_t}$ constructed above 
\begin{equation}
    \lim_{t \rightarrow 0}
    \frac{
        \E_Q(f_{\theta_t}, f_{\theta^*_{i}})
    }{
        \E_P(f_{\theta_t}, f_{\theta^*_{i}})^{1/\tilde{\rho}_{i}}
        }
    = \frac{t}{ (C t^{\rho_{i}})^{1/\tilde{\rho}_{i}}}
    = C' t^{1 - \frac{\rho_{i}} { \tilde{\rho}_{i}} }
    = \infty.
\end{equation}

Now that we have shown that $\rho_i$ are indeed minimal transfer exponents, we will show that for these minimal transfer exponents, $C_{\rho_i}$ is lower bounded by a function that depends only on $L+1$. Fix some $t_0 < \frac{1}{2(L+1)}$ and consider $f_{\theta_{t_0}}$.  By the same calculations as in \cref{eq: ex1 high error due to missing a region} we have $  \E_P(f_{\theta_t}, f_{\theta^*_{i}}) = \frac{2^{-(2 i \cdot \rho_i + i)}}{Z} t^{\rho_i}$ and 
\begin{align}
     \frac{
        \E_Q(f_\theta, f_{\theta^*_{i}})
    }{
        \E_P(f_\theta, f_{\theta^*_{i}})^{1/\rho_{i}}
        }
    = 
    \frac{t Z^{1/\rho_i}}{ 2^{-(2i + i/\rho_i)} t}
    = \frac{ Z^{1/\rho_i}}{ 2^{-(2i + i/\rho_i)} }.
\end{align}

Since $2^{-(2i + i/\rho_i)} \le 2^{-2i}$, it suffices to lower bound $Z^{1/\rho_i}$, note that $Z \ge 2^{-(2i\cdot \rho_i + i)} \paren{\frac{1}{L+1}}^{\rho_i} $ and plugging this into the expression above we get 
\begin{equation}
    \frac{
        \E_Q(f_\theta, f_{\theta^*_{i}})
    }{
        \E_P(f_\theta, f_{\theta^*_{i}})^{1/\rho_{i}}
        }
    \ge \frac{1}{L+1}.
\end{equation}

Finally, to see that $i^*_P = i^*_Q = L$, note that for every $1 \le i < L$, $f_{\theta^*_i}$ labels the interval $I_{i + 2}$, which starts at $v_{i  + 1}$, incorrectly, so it cannot achieve zero excess risk.
\end{proofof}

\begin{proofof}{ \cref{prop: realizing alternating signs}}

We will argue that there exists $a_\sigma \in \real^i$ and $r_\sigma \in \real$ such functions of the 
 
 \begin{equation}
     h_{a_\sigma, r_\sigma}(x) = \sgn{\sum_{j=1}^{i} \alpha_{\sigma, j} \; \sgn{ x - b'_j} + r_\sigma},
\end{equation}
 can produce the sign pattern $\sigma$. Since these functions are a restricted form of the two layer neural nets introduced in \cref{eq: def two layer threshold}, this would prove the first part of the claim.

Define the function $g: [0, 1] \mapsto \{\pm 1\}^{i+1}$ where 
\begin{equation}
    g(x)_j = \sgn{x - b'_j}.
\end{equation}
The function $g$ maps the $i + 1 $ intervals to $i + 1$ points on the unit cube. Let $x_1, \dots, x_{i+1}$ be a set of arbitrary points from each interval $I_1,\dots, I_{i+1}$, then we have 
\begin{align}
    & g(x_1) = [-1, -1, -1, \dots, -1],\\
    & g(x_2) = [+1, -1, -1, \dots, -1], \\
    & \dots \\
    & g(x_{i +1}) = [+1, +1, +1, \dots, +1].
\end{align}
Note that $g(x_1) = - g(x_2)$, but the set of vectors $g(x_1), \dots , g(x_i)$ are linearly independent.
The functions 
\begin{equation}
    h_{a, r}(x) = \sgn{a^\top g(x) + r}
\end{equation}
are  affine halfspaces parameterized by $a$ and $r$, so they can shatter any set of $i + 1$ points where there is at most two colinear points.
We take $a_\sigma$ and $r_\sigma$ to be coefficients of the affine halfspace that produces the labels $\sigma$.

To see that two layer neural nets of the form \cref{eq: def two layer threshold} parameterized by $\theta = (i, a, r, w, b)$ can have at most $i$ decision boundaries, note that adding a hidden unit can add at most one decision boundary, and when $i = 0$, there are no decision boundaries. 
\end{proofof}

\begin{proofof}{ \cref{claim: construction of risk minimizers for lower levels}}
Recall the regions $R_1, \dots, R_i$.
We first argue that if a classifier does not place a decision boundary in some region $i' \in [i]$, then its' error is larger than $\theta^*_i$, and then argue that among all the classifiers that place exactly one decision boundary in each of those regions, only the ones that have exactly the same decision boundaries as $\theta^{*, i}$ can be risk minimizers.
Suppose $\theta_i$ is some classifier that doesn't place any decision boundaries in $R_{i'}$, then it must mislabel one of the intervals to the left or right of $b^*_{i'}$, that is either the interval $I_{i'+1} \cap R_{i'}$ or $ I_{i'} \cap R_{i'}$. 
Then the risk can be lower bounded by

\begin{align} \label{eq: ex1 high error due to missing a region}
    \Risk_P(h_{\theta_i} ) &\ge P_X(I_{i'+1} \cap R_{i'}) 
    \\ &= 
    \int_{b^*_{i'}}^{\frac{v_{i'} + v_{i' + 1}}{2}}
    \frac{\rho_{i'} \cdot 2^{-(2 i' \cdot \rho_{i'} + i')}}{Z} \cdot \abs{x - b^*_{i'}}^{\rho_{i'} -1}  d x
    = \frac{ \rho_{i'} \cdot 2^{- (2 i' \cdot \rho_{i'} + i')}}{Z \rho_{i'}} \left( \frac{\size{I_{i'+1}}}{2} \right)^{\rho_{i'}} 
    \\ & = 
    \frac{2^{-(2 i' \cdot \rho_{i'} + i')}}{Z } \left( \frac{1}{2(L+1)} \right)^{\rho_{i'}}
    \\ & \ge 
    \frac{2^{-(2 i \cdot \rho_i + i)}}{Z } \left( \frac{1}{2(L+1)} \right)^{\rho_{i}}.
\end{align}
 
On the other hand, $\theta^*_{i}$ labels the first $ i + 1$ intervals correctly, so  
 \begin{equation} \label{eq: ex 1 risk of risk minimizer}
     \Risk_P(h_{\theta^*_{i}}) 
     \le  
     \sum_{j=i + 2 }^{L + 1}  P_X (I_j) .
 \end{equation}

Note that when $i = L$, by construction $\Risk_P(h_{\theta^*_{i}})  = 0$.

 For every $i +2 \le  j \le L $, we can write $ P_X (I_j) =  P_X(I_j \cap R_j) + P_X(I_j \cap R_{j-1}) = \frac{ 2^{-(2 j \cdot \rho_j + j)}}{Z} \left( \frac{1}{2(L+1)} \right)^{\rho_j} +
 \frac{ 2^{-(2 (j-1) \cdot \rho_{j-1} + j-1)}}{Z} \left( \frac{1}{2(L+1)} \right)^{\rho_{j-1}}
 $. Since for all $j$, $\rho_j \ge 1$, and
 $\rho_{j-1} \le \rho_j $,  we get that 
\begin{equation} \label{eq: upperbound on mass of interval under source}
    P_X (I_j) \le \frac{2^{-2 ((j-1) \cdot \rho_{j-1} + j-1) }}{Z } \paren{\frac{1}{L+1}}^{\rho_{j-1}}.
\end{equation}

Going back to \cref{eq: ex 1 risk of risk minimizer}, for $i < L$  we get
\begin{align}
    \Risk_P(h_{\theta^*_{i}}) 
     &\le \sum_{j=i + 2 }^{L + 1}
     \frac{2^{-(2 (j-1) \cdot \rho_{j-1} + j-1) }}{Z } \paren{\frac{1}{L+1}}^{\rho_{j-1}}
     \le 
     \frac{2^{-2 (i+ 1) \cdot \rho_{i + 1}}}{Z }
      \sum_{j=i + 2 }^{L + 1}
      2^{-j + 1} \cdot 
      \paren{\frac{1}{L+1}}^{\rho_{i + 1}}
     \\ & \le 
     \frac{2^{-(2 (i+ 1) \cdot \rho_{i + 1}  + i)}}{Z }
     \paren{\frac{1}{L+1}}^{\rho_{i + 1}}
     \\ & < 
    \frac{2^{-(2 i \cdot \rho_i + i)}}{Z } \left( \frac{1}{2(L+1)} \right)^{\rho_{i}} 
     \le 
     \Risk_P(h_{\theta_i} ),
\end{align}
so the excess risk can be lower bounded by 
\begin{align}
\Risk_P(h_{\theta_i} ) - \Risk_P(h_{\theta^*_{i}})  
    \ge 
    \frac{2^{-((2 i +1) \cdot \rho_i + i)} }{Z } \left( \frac{1}{L+1}\right)^{\rho_{i}}  \paren{1 - 2^{- \rho_{i + 1}}} 
    \ge
\frac{2^{-((2 i +1) \cdot \rho_i + i)}  }{ 2 Z } \left( \frac{1}{L+1} \right)^{\rho_{i}}  
\end{align}

and consequently for some $C_1  = \frac{2^{-((2 i +1) \cdot \rho_i + i)}  }{ 2 Z } $, 
\begin{equation}
    \label{eq: TNN erisk gap under source}
    \E_P(h_{\theta_i}, \hstar_{\theta_i}) \ge C_1 \left( \frac{1}{L+1} \right)^{\rho_{i}}  .
\end{equation}

Therefore, we have shown that any classifier $f_{\theta_i}$ that doesn't place a decision boundary in any of the first $i$ regions has a strictly larger excess risk than excess risk of  $f_{\theta^*_i}$. Then it would suffice to show that among the classifiers that place exactly one decision boundary in each of $R_1, \dots, R_i$, $f_{\theta^*_i}$ is a risk minimizer.

Let $\theta'_i$ be a function that places one decision boundary in each of the regions $R_1, \dots, R_i $. Let $v' \in [0,1] ^i$ denote the location of its' decision boundaries, and let $R_k$ be some region where $v'_k \neq b^{*}_k$, that is a region where the decision boundary is different from that of $\theta^{*}$.
Now we will argue that $\theta'_i$ cannot be a risk minimizer, since it can be modified to another classifier that has a strictly smaller risk. Without loss of generality assume that $ v'_k < b^{*}_k$, the other direction follows by the same argument. 
It must be that either $f_{\theta'_i}$ labels all the points in  the interval $(v'_k, b^*_k)$ incorrectly and is correct on the rest of $R_k$, or it mislabels at least all the points in $R_k \cap I_{k+1}$. In the latter case, as we have seen before in \cref{eq: ex1 high error due to missing a region}, since $k \in i$, 
excess risk of $f_{\theta'_i}$ would be strictly larger than excess risk of $f_{\theta^*_{i}}$.
If the only set of points that are mislabelled in $R_k$ are in the interval $(v'_k, b^*_k)$, then moving $v'_k$ right by increasing it to $b^*_k$ would eliminate error in this interval without affecting other intervals, thus strictly decreasing the 
excess risk. Therefore, any classifier that places decision boundaries not on 
$b^*_1, \dots, b^*_i$ cannot be a risk minimizer. 

Note that in terms of parameters, the risk minimizer is not unique, since there can be a family of parameters that give the same decision boundaries and signs. 
\end{proofof}

\begin{proofof} {\Cref{claim: te for TNN} }
Recall that he boundaries $v_1, \dots, v_L$ partitioned the unit interval into $L+1$  intervals $I_1, \dots, I_{L+1}$, where $ I_1 = [0, v_1], I_2 = (v_2, v_3]$ and so on. Recall that $\theta^*_{i}$ is the parameters of a risk minimizer in $\Hyp_i$, as described in \cref{claim: construction of risk minimizers for lower levels}. Based on the points in $V$, we also defined the regions $R_1, \dots , R_L$, where all the points in each region shared the same density function under $P$. Let $f_{\theta_i}$ be an arbitrary 
 element of $\Hyp_i$, by \Cref{prop: realizing alternating signs}, $f_{\theta_i}$ can have at most $i$ decision boundaries. We can break down the excess risk into the contributions from each region, 
 
 \begin{equation}
     \E_Q(f_{\theta_i}, f_{\theta^*_i}) = \sum_{j=1}^L Q_X ( f_{\theta_i}(X)\neq f_{\theta^*_i}(X) \land X \in R_j)
     = \sum_{j=1}^L Q_X(E_j),
 \end{equation}
where $E_j$ is the set of points in $R_j$ that $\theta_i$ labels differently from $\theta^*_i$. Consider the first $i$ regions, since $\theta_i$ has at most $i$ decision boundaries and the regions are disjoint, it must be the case that either 
\begin{enumerate}
    \item \label{case 1}$\theta_i$ does not place a decision boundary in at least one of the regions $R_1, \dots, R_i$, or
    \item \label{case 2} $\theta_i$ places exactly one decision boundary in every region $R_1, \dots, R_i$.
\end{enumerate}

We break down the proof into the two cases above, and start with the simpler case \ref{case 1}. In this case, there exists at least one region $R_{i'}$, for some $i' \le i$, such that $\theta_i$ has placed no decision boundary there and consequently the whole interval $R_{i'}$ has the same label, while under $\theta^*_{i}$, there would be a boundary at $v_{i'} \in R_{i'}$, which implies that $\theta_i$ must have mislabelled an interval on at least one side of $v_{i'}$. Then under source $P_X(E_{i'}) \ge   P_X(I_{i'+1} \cap R_{i'})$
and by \cref{eq: ex1 high error due to missing a region} and \cref{eq: TNN erisk gap under source} the excess risk
\begin{equation} \label{eq: ex1 lower bound on source case1}
     \E_P(f_{\theta_i}, f_{\theta^*_i}) = \sum_{j=1}^i P_X(E_j)
     \ge P_X(I_{i'+1} \cap R_{i'})  \ge C_1 \left( \frac{1}{L+1} \right)^{\rho_{i}} ,
\end{equation}
where $C_1$ is some positive constant that could depend on $i$ and $\rho_i$. 
On the other hand,
\begin{equation}
    \E_Q (f_{\theta_i}, f_{\theta^*_i}) 
    \le 1.
\end{equation}

Now by \cref{eq: ex1 lower bound on source case1}
\begin{equation}
     \E_P(f_{\theta_i}, f_{\theta^*_i})^{1/\rho_{i}}
     \ge \frac{ C_1^{1/\rho_{i}}}{L+1} = \frac{2^{- (2 i + 1) - i/\rho_i }}{(2 Z) ^{1/\rho_i} \cdot (L+1)}.
\end{equation} 

Using \cref{eq: upperbound on mass of interval under source}, we have that $Z \le (L+1) 2^{-4} \paren{\frac{1}{L+1}}^{\rho_L} \le 2^{-4}$, and consequently for any $\rho_i \ge 1$ and fixed $L$, 
and $\E_P(f_{\theta_i}, f_{\theta^*_i})^{1/\rho_{i}} \ge  \frac{2^{- 3 i - 1 }}{L+1}$
setting 
$C^{(1)}_{\rho_i} = \frac{L+1 }{2^{- 3 i - 1 }}$, we can conclude that in case \ref{case 1}
\begin{equation} \label{eq: transfer exponent case 1}
    \E_Q (f_{\theta_i}, f_{\theta^*_i}) \le C^{(1)}_{\rho_i}  \E_P(f_{\theta_i}, f_{\theta^*_i})^{1/\rho_{i}}.
\end{equation}

In case \ref{case 2}, when $i < L$, $\theta_{i}$ has no decision boundaries in regions $R_{i +1}, \dots, R_{L}$, so they will all  have the same label, since they also have the same label under $\theta^*_i$. it must be that either all their labels agree with those of $\theta^*_i$, or their label disagrees with the label $\theta^*_i$ assigns to those regions. If they are all labelled incorrectly, we will argue that,
\begin{equation} \label{eq: ex1 last section incorrect labels}
    P_X(E_i) \ge C_1 \left( \frac{1}{L+1} \right)^{\rho_{i}}.
\end{equation}
To see this, note that $\theta_i$ places its' decision boundary in region 
$R_i$ either to left of $b^*_i$ or to the right. In the former case, then 
 the interval $I_{i +1} \cap R_i$ is also labelled incorrectly, while in the latter case the interval $I_{i} \cap R_i$ would have incorrect labels, so in either case  by \cref{eq: ex1 high error due to missing a region} \cref{eq: TNN erisk gap under source} holds. Consequently, we can make the same arguments as in case $\ref{case 1}$ to get \cref{eq: transfer exponent case 1}.

 Going back to case \ref{case 2}, suppose that the regions $R_{i +1}, \dots, R_L$ are labelled according to $\theta^*_i$, or $i = L$ (they don't exist), so that they don't contribute to the excess risk. Let $ m \in \argmax_{j \in [i]} Q_X(E_j)$, so that $R_m$ is a region that has large contribution to the excess risk. Then 
 \begin{equation} \label{eq: ex1 case 2 target upper bound}
     \E_Q (f_{\theta_i}, f_{\theta^*_i})  = \sum_{j=1}^i Q_X(E_j)
     \le i \;  Q_X(E_m),
 \end{equation}
 while for the source, we can lower bound
 
 \begin{equation}
     \E_P(f_{\theta_i}, f_{\theta^*_i}) \ge P_X(E_m).
 \end{equation}

Let $b_m$ be the point that is a decision boundary in $R_m$ under $\theta_i$.
Then $E_m$ is either the interval that has $b^*_m, b_m$ as its' end points or it is a union of two intervals, one of which has size at least $\size{I_{m+1}}/2$. If it is a union of two intervals, since source excess risk will be bounded away from zero by a constant, then we can use the same arguments as in \cref{case 1}. If 
$E_m$ is an interval that has $b^*_m$ and $b_m$ as its' end points, then $P_X(E_m) = C_3 Q_X(E_m)^{\rho_m}$, where by similar calculations as those in \cref{eq: ex1 high error due to missing a region},  $C_3 =  \frac{2^{-(2 m \cdot \rho_{m} + m)}}{Z }$.
Since $Z \le 2^{-4}$ for any value of $ \rho_i$s and $L$, we have $ C_3^{1/\rho_m} \ge 2^{-2m  -m/\rho_m + 4/\rho_m}  \ge  2^{-3m}$.
Then 

\begin{equation}
    \E_P(f_{\theta_i}, f_{\theta^*_i})^{1/\rho_m} \ge P_X(E_m)^{1/\rho_m}
    = C_3^{1/\rho_m} Q_X(E_m), 
\end{equation}
and by \cref{eq: ex1 case 2 target upper bound} 
\begin{equation}
i  2^{3i} \cdot \E_P(f_{\theta_i}, f_{\theta^*_i})^{1/\rho_i}
\ge 
    \frac{i}{C_3^{1/\rho_m}}  \E_P(f_{\theta_i}, f_{\theta^*_i})^{1/\rho_m}
    \ge  \E_Q (f_{\theta_i}, f_{\theta^*_i}).
\end{equation}
Finally, setting $C_{\rho_i} = \max \braces{C_\rho^{(1)}, i  2^{3i} } = \max \braces{ (L+1)\cdot 2^{3i + 1}, i  2^{3i} }$, we have shown that $C_{\rho_i} \le (L+1) 2^{3i +1}$ is a transfer exponent with respect to $\Hyp_i$ .
\end{proofof}

\subsection{Proof of \Cref{prop: relu neural net}}

\begin{proofof} {\Cref{prop: relu neural net}}

    First we define a convenient parameterization of the set of all classifiers over the real line.
    \begin{definition} [Class of $k$ decision boundaries]
    Let $\Hyp^B_{k} \defeq\braces{h^B_{ b_1, \dots, b_k, \sigma _1 }}$ be the set of classifiers over $\real$ that have at most $k$ decision boundaries, given by points $b_1 < b_2 < \dots < b_k \in \real$, and  $\sigma_1 \in \braces{\pm 1}$, which is the sign of the first interval $(-\infty, b_1]$.
    \end{definition}

     Any two layer threshold neural net with $i $ activation units can have at most $i$ decision boundaries, so it belongs to $\Hyp^B_i$. \Cref{prop: realizing alternating signs}, shows that that the class of threshold neural nets with $i$ hidden units can generate $i$ boundaries, and all possible labellings of the corresponding intervals, so we can conclude that 
    $\Hyp_i = \Hyp^B_i$.

    Let $\functionClass^{CPWL}_k$ be the class of continuous piecewise linear (CPWL) functions with at most $k$ linear  pieces and consequently at most $k-1$ knots. 

    \begin{claim}
        We have that $\mathrm{sign} \circ \functionClass^{CPWL}_i = \Hyp^B_i$.
    \end{claim}
    \begin{proof}
    It is easy to see that once we fix $\sgn{0}$, thresholding each linear piece can result in at most one decision boundary, so a CPWL function with $i$ pieces can generate at most $i$ decision boundaries. 

    We argue that any set of $i$ decision boundaries $b_1< \dots < b_i$ and label assignment on the line can be generated by taking the sign of some CPWL function with at most $i$ pieces. To see this, consider the $i$ intervals $I_1 , \dots I_{i-1}$ whose end points are the boundaries. Let $m_1, \dots, m_{i-1}$ be the mid points of these intervals, and consider points $(m_1, \sgn{I_1}), \dots , (m_{i-1}, \sgn{I_{i-1}})$ on the $xy$-plane. 
    The CPWL function can be constructed by passing the first line through the pair of points $\paren{(b_1, 0), (m_1, \sgn{I_1})}$, the second line through $\paren{(b_2, 0), (m_2, \sgn{I_2})}$ and so on. The last line interpolates the points $(m_{i-1},\sgn{I_{i - 1}})$ and $(0, b_i)$.
    \end{proof}

    Following Lemma, which is adapted from \cite{Aliprantis2006} (Corollary 3.5),  states that any CPWL function with at most $k+1$ linear pieces can be written as a two layer ReLu Residual neural network with at most $k$ hidden units.

    This implies that $\mathrm{sign} \circ \functionClass^{CPWL}_{i+1} = \RLHyp_{i}$, since it is easy to see that any function of the form $f_\theta(x) = \left(\sum_{i=1}^{i}
        a_i [w_ix +b_i]_+ 
    \right) + \alpha x + r$ can have at most $i$ knots and consequently $i + 1$ linear pieces.

    \begin{lemma} [\cite{Aliprantis2006}, Corollary 3.5]
        Any CPWL function of the form 
        \begin{align}
        f(x) = 
            \begin{cases}
            m_0 x + c_0& \mbox{if} \; x \le b_0 \\
            m_i x + c_i & \mbox{if} \; b_{i-1} \le x \le b_i \;
            \mbox{for} \; 1 \le i \le k \\
            m_{k+1} x + c_k & \mbox{if} \; x \ge  b_k,
            \end{cases}
        \end{align}
        where $ - \infty < b_0 < b_1< \dots < b_k < \infty$ and $(m_i, c_i), 0 \le i \le k+1$ are real numbers,  can also be written in the form 
        \begin{align}
            f(x) = c_0 + m_0 x + \sum_{i= 0}^{k}(m_{i-1} - m_i) [t - b_i]_+.
        \end{align}
        
    \end{lemma}

\end{proofof}

\section{Example for $i^\sharp \doteq \arg\min_i \phi^\sharp(i)$ Below $i^*_P$} \label{app: i sharp example proof}

\begin{proposition}\label{prop: i sharp }
     Following up on Examples \ref{ex:NNthresholds} and \ref{ex:NN-Relu} with $L = 3$, 
    for every $ 1 \le \rho_1 \le \rho_2 \le \rho_3$, there exists $P$ and $Q$ over $[0, 1] \times \{\pm 1\}$  such that the following holds.  \rm{i}) $i^*_P = 3 $, \rm{ii}) $\rho_1, \rho_2, \rho_3 $ are minimal transfer exponents from $P$ to $Q$, where $C_{\rho_i}$'s are uniformly upper and lower bounded independently of $\rho_i$'s, and \rm{iii}) $\E_Q(\hstar_{P, 1}) = \E_Q(\hstar_{P, 2}) = \E_Q(\hstar_{P, 3}) $. 
    Consequently, while $i^*_P = 3$, we may choose $\rho_i$'s so that $i^\sharp\doteq \argmin_i \phi^\sharp(i)$ could be below any of the levels $1, 2, 3$, for $n_P$ sufficiently large. 

\end{proposition}

\begin{proof}
    We use the same construction as in \Cref{prop: threshold neural net}, with the exception that target does not share the same decision boundaries.
    Specifically, set $L = 3$ and let $v_1 , v_2, v_3$ be the decision boundaries under source. Under target, set $v_1' = \frac{v_1+ v_2} {2}$, $v_2' = \frac{v_2 + v_3}{2}$, and $v_3' = \frac{v_3 + 1}{2}$ to be the decision boundaries. 
    Let $I_1, \dots, I_4$, as defined in the proof of \Cref{prop: threshold neural net}, be the intervals defined by decision boundaries under source, and $I'_1, \dots ,I'_4$ be the intervals $[0, v_1'], [v'_1, v'_2], [v'_2, v'_3], [v'_3, 1]$. For any sequence of labels assigned to $I_1, \dots, I_4$, assign the same sequence of labels to $I'_1, \dots ,I'_4$.

    Since the marginal densities of $P$, $Q$ and $\hstar_{P, i}$ have not changed, the conditions on transfer exponent and coefficient  are satisfied.  Since $I_1 \subset I'_1$ and the sign patterns under source and target match, $\hstar_{P, i}$ don't make any errors under target in the interval $I_1$. 
    Exactly half of each of the intervals $I_2, I_3$ and $I_4$ is labelled $+1$ under target, and $\hstar_{P, i}$ give a single label to each of these intervals, so each $\hstar_{P, i}$ labels half of each of the intervals $I_2, I_3$ and $I_4$ incorrectly under target. 
    


    For the last part of the proposition, pick any value $i\in \{2, 3\}$, and to ensure that $\argmin_i \phi^\sharp(i) < i$, for all $j \le i$ set $\rho_j = \rho$ and for all $j > i$, set $\rho_j > \rho$.

\end{proof}

\section{Remaining Upper-bound Proofs}\label{app sec: remaining proofs for upper bounds}

Our analysis relies on the following lemma.
\begin{lemma} [\blue{\cite{VC:72}}]\label{lemma: uniform bound}
 Recall $\genGap{n_\dist}{\delta}{\classComp{\Hyp_i}} \defeq \frac{d_i \log (n_\dist/ d_i ) + \log 1/\delta}{n_\dist}$.  For any $\delta > 0$, with probability of at least $1 - \delta$, for all $h, h' \in \Hyp_i$ 
 
 \begin{align}
 \Risk_\dist (h) - \Risk_\dist(h') &\le \emRisk_\dist(h) - \emRisk_\dist(h') + 
      \sqrt{ \min \{\P_\dist[h \neq h'], \P_{n_\dist} [h \neq h'] \} \cdot  \genGap{n_\dist}{\delta}{\classComp{\Hyp_i}}} \\
       &\quad + 
       c \genGap{n_\dist}{\delta}{\classComp{\Hyp_i}}, \text{ and } \\
\frac{1}{2} \P_\dist[h \neq h'] &- c \genGap{n_\dist}{\delta}{\classComp{\Hyp_i}} 
    \le 
    \P_{n_\dist} [h \neq h'] 
    \le 2 \P_\dist[h \neq h']  
     + c \genGap{n_\dist}{\delta}{\classComp{\Hyp_i}}. \label{eq: lemma probabilities uniform}
\end{align}

\end{lemma}

\begin{proofof} {\Cref{lem: single level bounds.}}
Let $\hat{h}$ be the output of \Cref{alg: single level phi}, with $\tilde{\Hyp}^P \subseteq \hHyp_i^P$.
First we state a few useful claims. Proofs of these claims appear in \Cref{app sec: remaining proofs for upper bounds}
\begin{claim}\label{claim: upper bound on i hat Q}
    Let $\hat{i}_Q = \min i \;\;\;
            \mbox{s.t.}\;  \bigcap_{j \ge  i}^\infty \hHyp_j^Q \neq \emptyset$, 
    then with probability of at least $1- \delta$, $\hat{i}_Q \le i^*_Q$.
\end{claim}

Next claim can be used to bound the excess risk of $\hhat_{P, i}$ with respect to $\hstar_{P, i}$. Since a similar statement would also hold for $Q$ under the same high probability event as in  \cref{claim: upper bound on i hat Q}, and $\hat{i}_Q \le i^*_Q$, we could conclude that with probability of at least $1 - \delta$,
\begin{equation} \label{eq: target only excess risk upper bound.}
     \E_Q(\hat{h}_Q) \le c A(n_Q, \delta_{i^*_Q}, \classComp{\Hyp_{i^*_Q}})^{\frac{1}{2 - \beta_Q}}.
\end{equation}

\begin{claim} \label{claim: approximately ERM risk for one level}
     For any level $i$, and any $\hhat_i \in \hHyp_i^P$, with probability of at least $1- \delta_i$,
     \begin{equation}
        \E_P(\hhat_i, \hstar_{P, i}) \le c \genGap{n_P}{\delta_{i}}{\classComp{\Hyp_{i}}}^{\frac{1}{2 - \beta_{P, i}}}.
    \end{equation}
\end{claim}

Let $\event_Q$ and $\event_P$ be the events where the bounds in Claims \ref{claim: upper bound on i hat Q} and \ref{claim: approximately ERM risk for one level} hold. 
Let $\event_H$ be the event that the bounds given in \Cref{lemma: uniform bound} hold over the hypothesis class $\{\hat{h}_{P, i}, \hat{h}_Q\}$ and  held out samples from $Q$. Note that complexity of the class $\{\hat{h}_{P, i}, \hat{h}_Q\}$ is one.
    We first claim that under $\event_H$ and $\event_Q$, if $\hat{h} = \hhat_{P,i}$, then  $\E_Q(\hhat_{P,i}) \le  25  c \genGap{n_Q} {\delta} {\classComp{\Hyp_{i^*_Q}}}^{\frac{1}{2 - \beta_Q}}$.
     Suppose that $\hat{h} = \hhat_{P,i}$, which means that the if-statement condition in \Cref{alg: single level phi}  must have been satisfied. Under  $\event_H$, 
    \begin{align}
         \E_Q(\hhat_{P,i}) - \E_Q(\hat{h}_Q)  =& \Risk_Q(\hhat_{P,i}) - \Risk_Q(\hat{h}_Q)
         \le 
        \emRisk_Q'(\hhat_{P,i}) - \emRisk_Q'(\hat{h}_Q)
        \\ & +    \left(\P'_{n_Q}[\hhat_{P,i} \neq \hat{h}_Q] \cdot \genGap{n_Q}{\delta}{1}\right)^{1/2} + c \genGap{n_Q}{\delta}{1}
        \\ \le & 
        2 \left(\P'_{n_Q}[\hhat_{P,i} \neq \hat{h}_Q] \cdot \genGap{n_Q}{\delta}{1}\right)^{1/2} + 2c \genGap{n_Q}{\delta}{1}
        \\ \le &  2 \left(\P'_{n_Q}[\hhat_{P,i} \neq \hstar_Q] \cdot \genGap{n_Q}{\delta}{1}\right)^{1/2}
        \\ &+ 2 
        \left(\P'_{n_Q}[\hat{h}_Q \neq \hstar_Q] \cdot \genGap{n_Q}{\delta}{1}\right)^{1/2}
       + 
         2c \genGap{n_Q}{\delta}{1}.
    \end{align}
    
    By the second part of \Cref{lemma: uniform bound} and BCC,    
    \begin{align} \label{eq: source target risk difference}
         \E_Q(\hhat_{P,i}) - \E_Q(\hat{h}_Q) 
         \le &  
         2 \left( 2 \P_Q[\hhat_{P,i} \neq \hstar_Q] \cdot \genGap{n_Q}{\delta}{1}\right)^{1/2} \\
          & +
         2 
        \left( 2\P_Q [\hat{h}_Q \neq \hstar_Q] \cdot \genGap{n_Q}{\delta}{1}\right)^{1/2}
         + 
     4c \cdot \genGap{n_Q}{\delta}{1} \\
     \le & 
      2 \left( 2 C_{\beta_Q} \E_Q(\hhat_{P,i})^{\beta_Q}\cdot \genGap{n_Q}{\delta}{1}\right)^{1/2}
     \\ & +
     2 
    \left( 2 C_{\beta_Q} \E_Q (\hat{h}_Q)^{\beta_Q} \cdot \genGap{n_Q}{\delta}{1}\right)^{1/2}
     + 
     4c \cdot \genGap{n_Q}{\delta}{1}.
\end{align}

Assuming that $1 \le \classComp{\Hyp_{i^*_Q}}$, so that $ \genGap{n_Q}{\delta}{1} \le A(n_Q, \delta, \classComp{\Hyp_{i^*_Q}})$ and plugging in the bound in \cref{eq: target only excess risk upper bound.}, we can upper bound 
\begin{align}
     2 
    \left( 2 C_{\beta_Q} \E_Q (\hat{h}_Q)^{\beta_Q} \cdot \genGap{n_Q}{\delta}{1}\right)^{1/2}
      &\le
      8 c A(n_Q, \delta_{i^*_Q}, \classComp{\Hyp_{i^*_Q}})^{\frac{1}{2 - \beta_Q}}.
\end{align}

Now if $ 2 \left( 2 C_{\beta_Q} \E_Q(\hhat_{P,i})^{\beta_Q}\cdot \genGap{n_Q}{\delta}{1}\right)^{1/2} >  8 c A(n_Q, \delta_{i^*_Q}, \classComp{\Hyp_{i^*_Q}})^{\frac{1}{2 - \beta_Q}} \ge \E_Q(\hat{h}_Q) $, then going back to \cref{eq: source target risk difference},  we can upper bound 
\begin{align}
     \E_Q(\hhat_{P,i})
     &\le   8
     \left( 2 C_{\beta_Q} \E_Q(\hhat_{P,i})^{\beta_Q}\cdot \genGap{n_Q}{\delta}{1}\right)^{1/2}.
\end{align}

Solving for $\E_Q(\hhat_{P,i})$ gives the bound 
\begin{equation}
     \E_Q(\hhat_{P,i})
     \le C \genGap{n_Q}{\delta}{1}^{\frac{1}{2 - \beta_Q}}.
\end{equation}

On the other hand, if $ 2 \left( 2 C_{\beta_Q} \E_Q(\hhat_{P,i})^{\beta_Q}\cdot \genGap{n_Q}{\delta}{1}\right)^{1/2} \le  8 c \genGap{n_Q} {\delta_{i^*_Q}} {\classComp{\Hyp_{i^*_Q}}}^{\frac{1}{2 - \beta_Q}}$, then  the term $8 c \genGap{n_Q} {\delta_{i^*_Q}} {\classComp{\Hyp_{i^*_Q}}}^{\frac{1}{2 - \beta_Q}}$ dominates the r.h.s. of in \cref{eq: source target risk difference} and we get 
\begin{align}
     \E_Q(\hhat_{P,i}) 
     &\le
    24  c \genGap{n_Q} {\delta_{i^*_Q}} {\classComp{\Hyp_{i^*_Q}}}^{\frac{1}{2 - \beta_Q}}
    + 
    \E_Q(\hat{h}_Q) 
    \le  25  c \genGap{n_Q} {\delta_{i^*_Q}} {\classComp{\Hyp_{i^*_Q}}}^{\frac{1}{2 - \beta_Q}}.
\end{align}
In either case, if $\hat{h} = \hhat_{P,i}$ then
$
     \E_Q(\hhat_{P,i}) 
     \le
      25  c \genGap{n_Q} {\delta_{i^*_Q}} {\classComp{\Hyp_{i^*_Q}}}^{\frac{1}{2 - \beta_Q}},
$
or equivalently, if $\E_Q(\hhat_{P,i}) 
     >
      25  c \genGap{n_Q} {\delta_{i^*_Q}} {\classComp{\Hyp_{i^*_Q}}}^{\frac{1}{2 - \beta_Q}}$, then $\hat{h} = \hat{h}_Q$.
We can also argue that if $\E_Q(\hhat_{P,i}) \le \E_Q(\hat{h}_Q)$, then $\hat{h} = \hhat_{P,i}$. Suppose that $\E_Q(\hhat_{P,i}) \le \E_Q(\hat{h}_Q)$, then under the event $\event_H$, 
\begin{align}
     \emRisk_Q'(\hhat_{P,i}) - \emRisk_Q'(\hat{h}_Q) 
     \le & 
       \Risk_Q(\hhat_{P,i}) - \Risk_Q(\hat{h}_Q) 
      +    \left(\P'_{n_Q}[\hhat_{P,i} \neq \hat{h}_Q] \cdot \genGap{n_Q}{\delta}{1}\right)^{1/2} 
      \\ & + c \genGap{n_Q}{\delta}{1}
    \\  \le & 
    \left(\P'_{n_Q}[\hhat_{P,i} \neq \hat{h}_Q] \cdot \genGap{n_Q}{\delta}{1}\right)^{1/2} 
     + c \genGap{n_Q}{\delta}{1},
\end{align}
which means that the if-statement condition in \Cref{alg: upper bound} will be satisfied and $\hat{h} = \hhat_{P,i}$.

   We can then conclude that under the events $\event_H$ and $\event_Q$, 
\begin{equation}
    \E_Q(\hat{h}) \le  \min \left\{ \E_Q(\hhat_{P,i}),   25  c \genGap{n_Q} {\delta_{i^*_Q}} {\classComp{\Hyp_{i^*_Q}}}^{\frac{1}{2 - \beta_Q}}  \right\}.
\end{equation}

 Using the transfer exponent condition described in \Cref{def: transfer exponent model class} we would get 
\begin{equation} \label{eq: excess risk under source one level}
      \E_Q(\hhat_{P,i}) \le   \E_Q(\hstar_{P, i}) +   \E_Q(\hhat_{P,i}, \hstar_{P, i}) 
      \le   \E_Q(\hstar_{P, i}) + C_{\rho_{i^*_P}}  \E_P(\hhat_{P,i}, \hstar_{P, i})^{\frac{1}{\rho_{i^*_P}}}.
\end{equation}

Applying \Cref{claim: approximately ERM risk for one level},  under the event $\event_P$ 
\begin{equation}
    \E_P(\hhat_{P,i}) 
    \le
    C \genGap{n_P}{\delta_{i}}{\classComp{\Hyp_{i}}}^{\frac{1}{2 - \beta_{P, i}}}.
\end{equation}

Plugging this back into \cref{eq: excess risk under source one level}, we get
\begin{equation}
    \E_Q(\hhat_{P,i}) \le 
    \E_Q(\hstar_{P, i})
    +
     C \cdot C_{\rho_{i}}  
    \genGap{n_P}{\delta_{i}}{\classComp{\Hyp_{i}}}^{\frac{1}{(2 - \beta_P) \rho_{i}}}
\end{equation}

Finally, we can conclude that under events $\event_H, \event_Q$, and $\event_P$, which hold simultaneously with probability of at least $1 - 3 \delta $,

\begin{equation}
     \E_Q(\hat{h}) \le  \min \left\{ 
 \E_Q(\hstar_{P, i})
    +
     C \cdot C_{\rho_{i}}  
    \genGap{n_P}{\delta_{i}}{\classComp{\Hyp_{i}}}^{\frac{1}{(2 - \beta_P) \rho_{i}}}
    , 
    c A(n_Q, \delta_{i^*_Q}, \classComp{\Hyp_{i^*_Q}})^{\frac{1}{2 - \beta_Q}} \right\}.
\end{equation}    
\end{proofof}

\begin{remark}
    Notice from \cref{eq: source target risk difference} that instead of $\beta_Q$, we could have achieved the bound in terms of 
    $\min\braces{ \beta_{Q,i^*_P}, \beta_{Q,i^*_Q} } = \beta_{Q,\max\{i^*_P,i^*_Q\}}$.
\end{remark}

\begin{proofof}{\Cref{claim: upper bound on i hat Q}}
    Since the hypothesis classes are nested, $\hstar_Q \in \Hyp_j$ for every $ j \ge i^*_Q$. First, we argue that with probability of at least $1- \delta$, for every $j \ge i^*_Q$,
\begin{equation}
    \hstar_Q \in \hHyp_j^Q,
\end{equation}
which would then imply that $\bigcap_{j\ge i^*_Q}^\infty \hHyp_j^Q \neq \emptyset$, and consequently $\hat{i}_Q \le i^*_Q$. Let $ \hat{h}_{Q,j}$ be an empirical risk minimizer for level $j$. Let $\event_Q$ be the event where the bounds in \Cref{lemma: uniform bound} hold for every level of the hierarchy with $\delta_i = \delta w_i$ for each level, so that $\event_Q$ occurs with probability of at least $1 - \sum_{j = 1}^\infty \delta_j = 1- \delta$. Under $\event_Q$,  for every $j \ge i^*_Q$
\begin{align}\label{eq: uniform ratios target}
    \Risk_Q(\hat{h}_{Q, j}) - \Risk_Q(\hstar_Q) 
        \le & 
        \emRisk_Q(\hat{h}_{Q, j}) - \emRisk_Q(\hstar_Q) 
        +  
        \sqrt{ \P_{n_Q}[\hat{h}_{Q, j} \neq \hstar_Q] \cdot  \genGap{n_Q}{\delta_j}{\classComp{\Hyp_j}}} \\
         & + 
      \genGap{n_Q}{\delta_j}{\classComp{\Hyp_j}},
\end{align}
since $ \hstar_Q$ is a risk minimizer, moving the risk difference to the right hand side of the inequality and  the empirical risk difference to the left hand side gives
\begin{align}
    \emRisk_Q(\hstar_Q) - \emRisk_Q(\hat{h}_{Q, j}) 
    & \le 
        C \sqrt{ \P_{n_Q}[\hat{h}_{Q, j} \neq \hstar_Q] \cdot  \genGap{n_Q}{\delta_j}{\classComp{\Hyp_j}}} + c
      \genGap{n_Q}{\delta_j}{\classComp{\Hyp_j}}. 
\end{align}
Therefore, by \Cref{def : empirical level set}, under $\event_Q$, $  \hstar_Q \in \hHyp_j^Q$ for every $j \ge i^*_Q$, implying that $\hat{i}_Q \le i^*_Q$.
\end{proofof}

\begin{proofof} {\Cref{claim: approximately ERM risk for one level}}
Let $\event_P$ be the event that the bound in \Cref{lemma: uniform bound} holds over $P$ samples.
\begin{align}
    \E_P(\hat{h}_i, \hstar_{P, i})  = & \Risk_P(\hhat_i)- \Risk_P(\hstar_{P, i})
     \le  \label{eq: source empirical excess risk}
     \emRisk_P(\hat{h}_i) - \emRisk_P(\hat{h}_{P,i})  
       \\ & + \label{eq: source conf interval}
       \paren{ \P_{n_P} [\hat{h}_i \neq \hstar_{P, i}] \cdot  \genGap{n_P}{\delta_{i}}{\classComp{\Hyp_{i}}}}^{1/2} + c
      \genGap{n_P}{\delta_{i}}{\classComp{\Hyp_{i}}}.
\end{align}
Since $\hat{h}_i \in \hHyp_{i}^P$, the expression in \ref{eq: source empirical excess risk} can be upper bounded by 

\begin{equation}
    \left(\P_{n_P}[\hat{h}_i \neq \hat{h}_{P,i}] \cdot  \genGap{n_P}{\delta_{i}}{\classComp{\Hyp_{i}}}\right)^{1/2} + 
    c \genGap{n_Q}{\delta_{i}}{\classComp{\Hyp_{i}}}.
\end{equation}
Furthermore, under $\event_P$ we can upper bound
\begin{align}
    \P_{n_P}[\hat{h}_i\neq \hat{h}_{P,i}] 
    &\le
    \P_{n_P}[\hat{h}_i\neq \hstar_{P, i}] + 
    \P_{n_P}[\hstar_{P, i} \neq \hat{h}_{P,i}]
    \\ & \le 
    2\P_P[\hat{h}_i\neq \hstar_{P, i}] + 
    2\P_P[\hstar_{P, i} \neq \hat{h}_{P,i}] + c \genGap{n_P}{\delta_i}{\classComp{\Hyp_{i}}},
\end{align}

where the second inequality followed by applying the second part of  \Cref{lemma: uniform bound}, which is stated in equation \ref{eq: lemma probabilities uniform}.

By  Bernstein class noise condition (\Cref{assmp: Bernstein noise condition}), the first two terms above can be upper bounded by

\begin{equation}
    C  \E_P(\hat{h}_i, \hstar_{P, i})^{\beta_{P, i}} +  C  \E_P( \hat{h}_{P,i}, \hstar_{P, i})^{\beta_{P, i}},
\end{equation}

then going back to equation \ref{eq: source empirical excess risk}, we get
\begin{align}\label{eq: upper bound for source empirical excess risk}
    \emRisk_P(\hat{h}_i) - \emRisk_P(\hat{h}_{P,i}) 
     \le &
    C \left( \E_P(\hat{h}_i, \hstar_{P, i})^{\beta_{P, i}}\cdot  \genGap{n_P}{\delta_i}{\classComp{\Hyp_{i}}}\right)^{1/2}
    \\ & +
     C \left( \E_P(\hat{h}_{P,i}, \hstar_{P, i})^{\beta_{P, i}} \cdot  \genGap{n_P}{\delta_i}{\classComp{\Hyp_{i}}}\right)^{1/2} + c \genGap{n_P}{\delta_i}{\classComp{\Hyp_{i}}}.
\end{align}

Again , using  \Cref{assmp: Bernstein noise condition} and under $\event_P$, we can also upper bound the first term in equation \ref{eq: source conf interval}, 
\begin{align}
   \paren{ \P_{n_P} [\hat{h}_i \neq \hstar_{P, i}] \cdot  \genGap{n_P}{\delta_{i}}{\classComp{\Hyp_{i}}}}^{1/2}
    \le &  \paren{ \P_P [\hat{h}_i \neq \hstar_{P, i}] \cdot  \genGap{n_P}{\delta_{i}}{\classComp{\Hyp_{i}}}}^{1/2} \\
     & + c
      \genGap{n_P}{\delta_{i}}{\classComp{\Hyp_{i}}} 
     \\  \le 
    & C \left( \E_P(\hat{h}_i, \hstar_{P, i})^{\beta_{P, i}}\cdot  \genGap{n_P}{\delta_i}{\classComp{\Hyp_{i}}}\right)^{1/2} \\
    & +
      C \genGap{n_P}{\delta_{i}}{\classComp{\Hyp_{i}}}.
\end{align}

Note that this upper bound can be absorbed into the bound given in 
\ref{eq: upper bound for source empirical excess risk} by adjusting the constants. In total, we get
\begin{align}
    \E_P(\hat{h}_i, \hstar_{P, i})  \le & 
     C \left( \E_P(\hat{h}_i, \hstar_{P, i})^{\beta_{P, i}}\cdot  \genGap{n_P}{\delta_i}{\classComp{\Hyp_{i}}}\right)^{1/2}
    \\ & +
     C \left( \E_P(\hat{h}_{P,i}, \hstar_{P, i})^{\beta_{P, i}} \cdot  \genGap{n_P}{\delta_i}{\classComp{\Hyp_{i}}}\right)^{1/2} + c \genGap{n_P}{\delta_i}{\classComp{\Hyp_{i}}}.
\end{align}
      
Now since $\hat{h}_{P,i}$ is an ERM over the class $\Hyp_{i}$ under $P$, 
under  \Cref{assmp: Bernstein noise condition} and event $\event_P$, using \cref{lemma: uniform bound} we can upper bound its' 
excess risk by 
\begin{equation}
    C \left(\genGap{n_P}{\delta_{i}}{\classComp{\Hyp_{i}}}\right)^{\frac{1}{2 - \beta_{P, i}}
    },
\end{equation}
which leads to the upper bound
\begin{align}
    \E_P(\hat{h}_i,  \hstar_{P, i}) \le &
     C \left( \E_P(\hat{h}_i, \hstar_{P, i})^{\beta_{P, i}}\cdot  \genGap{n_P}{\delta_i}{\classComp{\Hyp_{i}}}\right)^{1/2}
    \\ & +
     C \left(\genGap{n_P}{\delta_{i}}{\classComp{\Hyp_{i}}}\right)^{\frac{1}{2 - \beta_{P, i}}
    }.
\end{align}

Consider the inequality above without the term on the second line, we can then solve for $\E_Q(\hat{h}_Q)$ and get the bound in the statement of this claim, since the solution will in the same order as the second line. 
\end{proofof}

\section{Remaining Lower Bound Proofs} \label{app sec: lower bound proofs}
\subsection{Proof of \Cref{thm: basic lower bound}}

The proof builds on Theorem 1 in \cite{hanneke2019value}, simply by enriching the family of distributions used therein.  
Let $ \varepsilon_Q \defeq \left( \frac{d}{n_Q} \right)^{\frac{1}{2 - \beta_Q}}$ and 
$\varepsilon_P \defeq \left(\frac{d}{n_P} \right)^{\frac{1}{(2- \beta_P)\rho}}$,
note that
    \begin{align}
        \min \left \{ \varepsilon_P+ \alpha , \varepsilon_Q \right \}
         & \ge 
        \min \left \{  \max \left \{
        \varepsilon_P , \alpha \right\} , \varepsilon_Q \right \}
         = 
        \max \left \{  \min \left \{
        \varepsilon_P , \varepsilon_Q \right\}
        ,\min \left \{
        \alpha , \varepsilon_Q \right \} \right \},
    \end{align}
    where the equality follows by distributing the $\min$.
Let $c_1, c_2 \le 2$ be constants that will be determined later; define 
    \begin{equation}
        \epsilon_1  \defeq c_1 \min \left \{
        \varepsilon_P, \varepsilon_Q \right\} ,
    \end{equation}
    and 
    \begin{equation}
        \epsilon_2  \defeq c_2 \min \left \{
        \alpha , \varepsilon_Q \right \}.
    \end{equation}
    
    Theorem 1 of \cite{hanneke2019value} gives a lower bound of order 
    $c \epsilon_1$ which holds with probability of at least 
    $\frac{3 - 2 \sqrt{2}}{8}$, for some universal constant $c$. 
    Here, we will construct another family of distributions that would lead to a lower bound of order $\Tilde{c}\epsilon_2$ for a universal constant $\Tilde{c}$. In fact, the only difference is that 
    the source distribution is the same for all the members of hard family of distributions.

    Source and target marginal distributions are supported on a set of points $x_0, x_1, \dots, x_{d-1}$ in the domain $\X$ that is shattered by $\Hyp$. 
    Only the target distribution in the family of hard distributions $\{P^{n_P} \times Q_\sigma^{n_Q}\}$ depends on $\sigma \in \{\pm 1\}^{d-1}$. Source marginal distribution $P_X$ is the uniform distribution on $x_0, x_1, \dots x_{d-1}$, and $P_{Y \mid X = x_i}(Y =  1) = 1$. 
    For the target, let $Q_{X, \sigma}(x_0) = 1 - \epsilon_2^{\beta_Q}$ and $Q_{X, \sigma}(x_i) = \frac{\epsilon_2^{\beta_Q}}{d-1}$for $i \ge 1$. The target labels  for $i \ge 1$ are given by $Q_{\sigma, Y \mid X = x_i}(Y= 1) = \frac{1}{2} + \frac{\sigma_i}{4} \cdot \epsilon_2^{1- \beta_Q}$, and $Q_{\sigma, Y\mid X = x_0}(Y=1)= 1$.

    Now we can verify the Bernstein class noise condition. Let $h_\sigma \in \Hyp$ be the Bayes classifier under $Q_\sigma$ and let $\hdist(.,.)$ denote number of coordinates $\sigma, \sigma'$ differ, or equivalently Hamming distance of $\frac{\sigma + 1^{d-1}}{2}$ and $\frac{\sigma' + 1^{d-1}}{2}$ 
    Note that for any distinct pair $\sigma, \sigma' \in \{\pm 1\}^{d-1}$, 
    \begin{equation} \label{eq: from hamming distance to excess risk}
        \E_{Q_{\sigma'}} (h_\sigma ) = \hdist(\sigma', \sigma) \cdot 
        \frac{\epsilon_2^{\beta_Q}}{2 (d-1)} \cdot \epsilon_2^{1- \beta_Q} = 
        \frac{\hdist(\sigma', \sigma)}{2(d-1)} \cdot \epsilon_2,
    \end{equation}
    while 
    \begin{equation}
        \P_{Q_{\sigma'}} (h_{\sigma} \neq h_{\sigma'}) = \hdist(\sigma', \sigma) \cdot 
        \frac{\epsilon_2^{\beta_Q}}{d-1}.
    \end{equation}
Additionally, for every $\sigma'$, $\E_{Q_{\sigma'}}(h^*_P) \le \epsilon_2/2  \le c_2/2 \cdot \alpha$.  

By Proposition 5 of \cite{hanneke2019value}, there exists a $(d-1)/8$ packing $\mathcal{N}$ of the $d-1$ dimensional hypercube such that the all ones vector $1^{d-1} \in \mathcal{N}$ and $\size{\mathcal{N}} \ge 2^{\frac{d-1}{8}}$, where the metric used for the packing is Hamming distance.  Suppose that $\mathcal{N}'$ is such a packing over the $\{\pm 1\}^{d-1}$ hypercube. Now consider the restriction of the family of distributions to $\sigma \in \mathcal{N}'$, by \cref{eq: from hamming distance to excess risk}, for every distinct $\sigma, \sigma' \in \mathcal{N}'$, $\E_{Q_\sigma}(h_{\sigma'}) \ge  \frac{\epsilon_2}{16}$. 

Next, we show that the KL divergence between distributions parameterized by any two distinct $\sigma, \sigma'$ is small. First, write

\begin{align} \label{eq: KL bound}
    \KLDiv{P^{n_P} \times Q_{\sigma}^{n_Q}}{P^{n_P} \times Q_{\sigma'}^{n_Q}} = n_Q \KLDiv{Q_{\sigma}}{\Q_{\sigma'}} 
    = n_Q \sum_{i=1}^{d-1} \frac{\epsilon_2^{\beta_Q}}{d-1} \cdot 
        \KLDiv{Q_{\sigma, Y \mid X = x_i}}{Q_{\sigma', Y \mid X = x_i}}.
\end{align}
    Now we use Lemma 2 in \cite{hanneke2019value}, which gives an upper bound on KL divergence of two Bernoulli distributions with small bias, to get that $ \KLDiv{Q_{\sigma, Y \mid X = x_i}}{Q_{\sigma', Y \mid X = x_i}} \le c_0 \epsilon_2^{2 - 2 \beta_Q}$, as long as $\epsilon_2^{1- \beta_Q} < 1/2$. Going back to \cref{eq: KL bound}, we get 
    \begin{align}
        \KLDiv{P^{n_P} \times Q_{\sigma}^{n_Q}}{P^{n_P} \times Q_{\sigma'}^{n_Q}} \le 
        n_Q \cdot \epsilon_2^{\beta_Q} \cdot c_0/4 \cdot \epsilon_2^{2 - 2\beta_Q} \le c_0/4 \cdot  n_Q  \cdot \epsilon_2^{2 - \beta_Q}
        \le c_0/4 \cdot c_2^{2- \beta_Q} \cdot d. 
    \end{align}
Now pick $c_2$ such that $c_0/4 \cdot c_2^{2- \beta_Q} \cdot d < 1/8 \log (\frac{d-1}{8})$, so that we can apply Proposition 4 of \cite{hanneke2019value} (which is Theorem 2.5 of \cite{tsybakov2009introduction}) to get that 

\begin{equation}
    \sup_{\sigma \in \mathcal{N}} \P_{P^{n_P} \times Q_\sigma^{n_Q}}
    \left[ \E_{Q_\sigma}(\hat{h}) \ge 1/32 \cdot \epsilon_2  \right] 
    \ge \frac{3 -2 \sqrt{2}}{8}.
\end{equation}

\subsection{Proof of \Cref{thm: lower bound for adaptivity to erisk across levels}}

Let $x_0, x_1$  be the set of points that are exclusively shattered by $\Hyp_2$, then it is possible to pick $h_1 \in \Hyp_1$ and $h_2 \in \Hyp_2 \setminus \Hyp_1$ such that $h_1$ and $h_2$ disagree on exactly one of $x_0, x_1$.   
Without loss of generality assume that $h_1(x_0)= h_2(x_0) = y_0$ and $h_2(x_1) \neq h_1(x_1)$. Since there is no label noise, $\beta_{P, i} = \beta_{Q, i} = 1$, for $i \in \{1, 2\}$.

Source distribution does not depend on $b$, so we have $P_{X, \sigma}(x_0) = P_{X, \sigma} (x_1) = 1/2$, and the labels are given by $h_2$, 
Target marginal distribution also does not depend on $b$, and is given by $Q_{X, \sigma}(x_0) = 1 - \alpha$, and $Q_{X, \sigma}(x_1) = \alpha$. The labels for target are set so that when $\sigma =  1$, $h_1$ is a risk minimizer, and when $\sigma= 2$, $h_2$ is a risk minimizer. That is, $Q_{Y \mid X = x, \sigma}(1) = h_\sigma(x)$.
Note that for every $\sigma, i \in \{1, 2\}$
\begin{align} \label{eq: max q risk of optimal source hypothesis}
    \E_{Q_\sigma}(\hstar_{P,i}) =
    \begin{cases}
        0 &    \sigma =  i \\
        \alpha & b\neq i.
    \end{cases}
\end{align}

For any classifier $\hat{h}$, define
\begin{align}
    \hat{\sigma} \defeq
    \begin{cases}
        1 & \hat{h}(x_1) = h_1(x_1)\\
        2 & \hat{h}(x_1) = h_2(x_1).
    \end{cases}
\end{align}

Let $\event_\sigma$ be the event where under target all the samples are $(x_0,y_0)$, then for any $\sigma$
\begin{align}
        \Prob_{P_\sigma^\np \times Q_\sigma^\nq}\brackets{\event_B} = \paren{1 - \alpha}^\nq \ge  
        \paren{1 - \frac{1}{c_1\nq}}^\nq
        \ge 1 - 1/c_1.
\end{align}

Under the event $\event_B$, $\hat{h}$ cannot distinguish between $Q_1$ and $Q_2$. So under the event $\event_B$, no classifier can output the correct answer more than half the times, so with probability of at least $\frac{1}{2} \cdot \paren{1 - 1/c_1}$ 
\begin{equation}
    \E_{Q_\sigma}(\hat{h}) \ge \E_{Q_\sigma}(h_{\hat{\sigma}}) = \E_{Q_\sigma}(h_{P, \hat{\sigma}}) = \max_i \braces{ \E_{Q_\sigma}(\hstar_{P, i})}.
\end{equation}
 Setting $c_1 = 2$ proves the statement.

\subsection{Proof of \Cref{prop: adaptivitiy bad distributions}}

\paragraph{Proper estimators.}
Let $\hat{h} \in \Hyp$ be some proper estimator. If $\hat{h} \in \{h_1, h_2\}$, set $\sigma_1 = +1$ and $\sigma_2$ arbitrary.
If $\hat{h}\in \{h'_1, h'_2\}$, then set $\sigma_1 = -1$, and $\sigma_2$ such that the region that has mass $(\frac{1}{c_1 n_P})^{1/\rho_a}$ is labelled incorrectly. That is, if $\hat{h} = h'_1$, $\sigma_2 = +1$. It is easy to see that with this choice of $\sigma$, $\E_{Q_{\sigma}} (\hat{h}, h^*_{Q_{\sigma}}) \ge (\frac{1}{32 n_P})^{1/\rho_a}$.

\paragraph{Improper estimators.}
Let $\hat{h}$ be an improper estimator.
For $-1/2 \le \epsilon \le 1/2$, we say that $\hat{h}$ has bias $\epsilon$ on an interval $I$ if it classifies $1/2 + \epsilon$ fraction of the interval under uniform measure as positive. That is, $\exf{U(I)}{\hat{h}(X)} = 2 \epsilon$.
Note that if a classifier has bias $\epsilon$ on an interval $I$, and $\sgn{\epsilon} \neq \sgn{I}$, then the error of the classier on that interval is $1/2 + \abs{\epsilon}$. 
Even if $\sgn{\epsilon} = \sgn{I}$, as long all of the interval has the same label, the error will be at least $1/2 - \abs{\epsilon}$.

For simplicity let $a \defeq  (\frac{1}{c_1 n_P})^{1/\rho_a}$, $b \defeq (\frac{1}{c_1 n_P})^{1/\rho_b}$, and recall that $\Delta = a - b$. Also note that in our construction, the risk minimizer has risk equal to $a$ under $\sigma = (+1, \sigma_2)$, while the risk minimizer under $\sigma = (-1, \sigma_2)$ has risk equal to $ \Delta/2$.

Now fix some improper estimator $\hat{h}$, and let $\el, \elin, \erin, \er$ be biases of $\hat{h}$ on ${\Lo, \Li, \Ri}$ and ${\Ro}$ respectively. We break down the proof to three cases. 

\textbf{Case 1: } $\elin, \erin \le -1/4$.

Let $\sigma_1 = +1$ so that the intervals $\Li, \Ri$ have positive labels and $\hat{h}$ has risk of at least $3/ 4 (a + b)$ on the intervals $\Ri, \Li$. 
We can pick $\sigma_2$ such that the error on the region $\Lo \cup \Ro$ is at least 
\begin{equation}
    \paren{(\frac{1}{2}  + \max (\abs{\el},\abs{\er})) + (\frac{1}{2} - \min(\abs{\el}, \abs{\er} )) } \cdot \frac{\Delta}{2} \ge \frac{\Delta}{2},
\end{equation}
by making sure that $\hat{h}$ makes more error on the $\Lo$ or $\Ro$ interval that has the maximum absolute bias.
Then for this $\sigma$, $\Risk_{\sigma}(\hat{h}) \ge \frac{a - b}{2} + \frac{3(a + b)}{4} \ge \frac{5a}{4}$, while the risk minimizer has error of $a$, so $\E_{Q_\sigma}(\hat{h}) \ge \frac{a}{4}$.

\textbf{Case 2:}$\abs{\el + \er} \ge 1/4$, and Case 1 condition does not hold.

Set $\sigma_1 = -1$, and pick $\sigma_2$ such that $\sigma_2 = - \sgn{\el + \er}$. Note that total bias over the region $\Lo \cup \Ro$ would be $\frac{\er + \el}{2}$, since  $ \paren{1/2 + \el}\cdot \frac{\Delta}{2} + \paren{1/2 + \er} \cdot \frac{\Delta}{2} = \paren{\frac{1}{2} + \frac{\er + \el}{2}} \cdot \Delta$. 
On the other hand, since we are in case 2, it must be that either $\elin > -1/4$ or $\erin > -1/4$, which would mean that the error over the intervals $\Li$ and $\Ri$ is at least $1/4 b$

Then we can ensure that 
\begin{equation}
    \E_{Q_\sigma}(\hat{h}) \ge  \paren{\frac{1}{2} + \frac{\abs{\er + \el}}{2}} \cdot (a - b)  +  \frac{b}{4} - \frac{a- b}{2}  
    \ge 
    \frac{a-b}{8} + \frac{b}{4}
    \ge \frac{a}{8}.
\end{equation}

\textbf{Case 3:}$\abs{\er + \el} < 1/4$ and the condition in Case 1 does not hold.

Set $\sigma_1 = -1$, and pick $\sigma_2$ such that whichever of $\Ri$ or $\Li$ that has more positive bias is assigned mass $a$. Since we are not in Case 1,  $\max(\erin, \elin) > -1/4$, leading to error of at least $\frac{a}{4}$ over $\Ri \cup \Li$. 
On the other hand, since the bias in the regions $\Ro$ and $\Lo$ is $\frac{\er + \el}{2}$, we have 

\begin{equation}
    \E_{Q_\sigma}(\hat{h}) \ge \paren{\frac{1}{2} - \frac{\abs{\er + \el}}{2}} \cdot (a-b) + \frac{a}{4}
    - \frac{a - b}{2} 
    \ge \frac{a}{8}.
\end{equation}

The statement of the proposition follows by lower bounding $ \paren{\frac{1}{32}}^{1/\rho_a} \ge \frac{1}{32}$.

\section{Adaptivity Lower Bounds for a Larger Class} \label{sec: adaptivity lower bounds for a larger model class}
In this section, restricting to proper learners, we show similar adaptivity lower bounds as in \Cref{thm: simple rho lower bound } for a larger model class.
Let $\bar{\Hyp}_1= \{h_t\}$ be the class of one sided thresholds, where $h_t(x) = \sgn{x - t}$. Let $\bar{\Hyp}_2$ additionally include one sided intervals, where only the points inside a closed interval are labelled positive. 

\begin{theorem}\label{thm: extended adaptivity lower bound}
  Let $\bar{\Hyp}_1$ and $\bar{\Hyp}_2$ be the class of one sided thresholds and intervals as described above. Pick any 
$\rho_a > \rho_b \geq 1$, and any $n_P$ and $n_Q$, where   $ \left( \frac{1}{32 n_P}\right)^{1/\rho_a} \le \min \braces{ \frac{1}{24}, \frac{1}{32 n_Q}}$. There exists a family of distributions 
    $\{ \paren{P_\sigma, Q_\sigma}\}$, indexed by some $\sigma$, such that the following hold. 
\begin{enumerate}[leftmargin=14pt]
        \item[(i)] For all $\sigma $,   minimal transfer exponents from $P_\sigma$ to $Q_\sigma$
        are the set $\{\rho_1,\rho_2\} = \{\rho_a,\rho_b\}$.

         \item[(ii)] For all $\sigma$, we have 
        $\min_{i} \phi_\flat(i) = \left( \frac{1}{n_P}\right)^{1/\rho_b}$, strictly less than $\max_{i} \phi_\flat(i) = \left( \frac{1}{n_P}\right)^{1/\rho_a}$.
    
\end{enumerate} 
    \begin{align}
        \text{We have that, } \forall \hat h, \quad \sup_\sigma \Prob_{P_\sigma^\np \times Q_\sigma^\nq} \left[  \E_{Q_\sigma} (\hat{h}) \ge \frac{1}{64}\cdot \max_{i} \phi_\flat(i)\right] \ge 1/8 .
    \end{align}
\end{theorem}

\begin{proofof} {\Cref{thm: extended adaptivity lower bound}}

In this proof, since the construction is very similar to the one in \Cref{thm: simple rho lower bound }, we will use the same notation and refer to the objects defined there.

\textbf{The family of distributions}.
We divide the unit interval as in the proof of \Cref{thm: simple rho lower bound } and let $\sigma \in \{\pm 1\}$ . Recall $\Delta \defeq \left(\frac{1}{c_1 n_P}\right)^{1/\rho_a} - \left(\frac{1}{c_1 n_P}\right)^{1/\rho_b}$,
where $c_1$ is constant that will be picked later . Source distributions $P_\sigma$ are the same as in the construction in \Cref{thm: simple rho lower bound }. The target marginals are as follows. 
\begin{itemize}
    \item $Q_{X, (1, 1)} (\Lo) = 0$, and $Q_{X, (1, 1)}(\Ro) =  \Delta$.
    \item $Q_{X, (1, -1)} (\Lo) = \Delta$, and  $Q_{X, (1, 1)}(\Ro) =  0$.
    \item $ Q_{X, (-1, \cdot)} (\Lo) =  Q_{X, (-1, \cdot)} (\Ro) = \Delta/2 $.
    \item $Q_{X, (\cdot, +1)} (\Li) =  \left(\frac{1}{c_1 n_P}\right)^{1/\rho_a}$
    and $Q_{X, (\cdot, +1)} (\Ri) =  \left(\frac{1}{c_1 n_P}\right)^{1/\rho_b}$.
    \item $Q_{X, (\cdot, -1)} (\Li) =  \left(\frac{1}{c_1 n_P}\right)^{1/\rho_b}$
    and $Q_{X, (\cdot, -1)} (\Ri) =  \left(\frac{1}{c_1 n_P}\right)^{1/\rho_a}$.
    \item The remaining mass is in the middle interval, so $Q_X([1/3 + r, 2/3 - r]) = 1 - 2 \left(\frac{1}{c_1 n_P}\right )^{1/\rho_a}$.
\end{itemize}

The masses in all intervals except for $\Ri$ and $\Li$ are distributed uniformly within that interval. For intervals $\Li$ and $\Ri$, the densities are 
\begin{itemize}
    \item  $\den_L(x) \propto \abs{x - (1/3 + r/2)}^{\frac{1}{\rho_a} - 1}$ and $\den_R(x) \propto \abs{x - (2/3 -r/2)}^{\frac{1}{\rho_b} - 1}$ when $\sigma_2 = +1$, and 
    \item $\den_L(x) \propto \abs{x - (1/3 + r/2)}^{\frac{1}{\rho_b} - 1}$ and $\den_R(x) \propto  \abs{x - (2/3 -r/2)}^{\frac{1}{\rho_a} - 1}$ if $\sigma_2 = -1$.
\end{itemize}
In this construction, only the labels of the intervals $\Ri$ and $\Li$ depend on $\sigma$, and are given by $Y_{Q,\sigma}(\Li) = Y_{Q,\sigma}(\Ri) = \sigma_1$. If the intervals $\Lo$ and $\Ro$ have non zero mass under $\sigma$, then they are labelled $+1$. The middle interval $[1/3 + r, 2/3 - r]$ is labelled $-1$ for every $\sigma$.

\begin{claim}
    Recall $\Hyp_1 \subset \Hyp_2$ from \Cref{thm: simple rho lower bound }. For every $\sigma$ and $i \in \{1, 2\}$, we have $\E_{Q_\sigma}(\hstar_{P_\sigma, i}) = 0$ and the risk minimizers over the classes  $\Hyp_1$ and $\Hyp_2$ under both source and target are the same as the risk minimizers over classes $\bar{\Hyp}_1, \bar{\Hyp}_2$.
\end{claim}
\begin{proof}
    Since the middle interval has a large negative mass and $\bar{\Hyp}_1$ is the class of one sided thresholds, any one sided threshold that positively labels the middle interval cannot be a risk minimizer. Since the threshold is in the intervals $\Li \cup \Lo$, we can see that the risk minimizers are either $h_1$ or $h_1'$ and are shared between source and target, implying that $\E_{Q_\sigma}(\hstar_{P_\sigma, 1}) = 0$.

    Under source a one sided interval that is a risk minimizer would choose to label intervals $\Lo, \Li$ accurately, since there is large negative mass in the middle interval, and the mass in $\Ro$ is small than the mass in $\Lo$ by a constant factor.  Under target, there are multiple one sided intervals that are risk minimizers, but since the total positive mass in the left side ( $\Lo \cup \Li$) is equal to the total positive mass in the right side $\Ro \cup \Ri$, and the negative mass in the center interval is very large, one of $h_2$ or $h'_2$ would also be a risk minimizer under target depending on $\sigma_1$, and it would be shared with source, so $\E_{Q_\sigma}(\hstar_{P_\sigma, 2}) = 0$.
    
\end{proof}

\begin{claim}
     For every $\sigma = (\sigma_1, \sigma_2) \in \{\pm 1\}^2$, if $\sigma_2 = 1$, then $\rho_b$ and $\rho_a$ are transfer exponents from $P_\sigma$ to $Q_\sigma$ with respect to $\bar{\Hyp}_1$ and $\bar{\Hyp}_2$ respectively. If $\sigma_2 = -1$, then they are transfer exponents with respect to $\bar{\Hyp}_2$ and $\bar{\Hyp}_1$ instead.
\end{claim}

\begin{proof}
    To see that $\rho_a$ and $\rho_b$ are transfer exponents, note that the labels are always the same under source and target, and the only intervals where ratio of densities of source and target is not a constant are $\Li$ and $\Ri$. In the case of one sided thresholds, if some $h_t \in \bar{\Hyp_1}$ has source excess risk that is $ \epsilon < \frac{1}{c_1 n_P}$, it must be that $t \in \Ri$. Which then implies that its' target excess risk is going to be of order $c \left(\frac{\epsilon}{c_1 n_P}\right)^{1/\rho_b} $ or $c \left(\frac{\epsilon}{c_1 n_P}\right)^{1/\rho_a} $ depending on $\sigma_2$. 
    Similarly, any $h \in \bar{\Hyp_2}$ that has source excess risk $ \epsilon < \frac{1}{c_1 n_P}$ must be a once sided interval with both of its' end points in the region $\Lo \cup \Li$. 
    If the region that it makes error on is not in $\Li$, then the ratio of source and target excess risks is bounded by a constant, while if the error region is in $\Li$, $h$
    will have excess risk of order $\left(\frac{\epsilon}{c_1 n_P}\right)^{1/\rho_b} $ or $c \left(\frac{\epsilon}{c_1 n_P}\right)^{1/\rho_a} $ depending on $\sigma_2$.

    To argue that $\rho_a$ and $\rho_b$ are minimal transfer exponents, fix $\sigma_2 = -1$ and consider a sequence of one sided thresholds $h_{2/3 - r + t}$ as $t \rightarrow 0$. Target 
    excess risk for this sequence decreases at the rate $\left(\frac{t}{c_1 n_P}\right)^{1/\rho_a}$, while under 
    source it would be $\frac{t}{c_1 n_P}$. If $\rho' < \rho_a$ is a transfer exponent, the ratio of the excess risks $\frac{t^{1/\rho_a}}{t^{1/\rho'}}$ would not be bounded by a constant as $t \rightarrow 0$.
    A similar argument works for $\bar{\Hyp_2}$ and $\sigma_2 = +1$, since $\hstar_{P, 2} \in \Hyp_2$.
\end{proof}

 Next, we show that for every proper learner $\hat{h} \in \bar{\Hyp_2}$, there is a  distribution in the family where $\hat{h}$ incurs large excess risk.

 \begin{proposition}
     Let $c_1 = 32$. For any proper learner $\hat{h}$, there exists $\sigma \in \{\pm 1\}^2$ such that $ \E_{Q_\sigma}(\hat{h}) \ge 1/2 \cdot \left(\frac{1}{c_1 n_P}\right)^{1/\rho_a}$
 \end{proposition}
 \begin{proof}
     By construction, for every proper learner $\Tilde{h} \in \Hyp_2$, there exists $\sigma$ such that $\E_{Q_\sigma}( \tilde{h}) \ge \left(\frac{1}{c_1 n_P}\right)^{1/\rho_a}$. We project every proper learner $\hat{h} \in \bar{\Hyp_2}$ by picking $h \in \Hyp_2$ whose labeling on the regions $\Ri$ and $\Li$ agrees the most with $h$, under the uniform measure over $\Li$ and $\Ro$. In the case that $\hat{h}$ has positive labels in both of the regions, its' excess risk will be a large constant. So $\hat{h}$ agrees with its' projection $h$ on at least one of the intervals $\Ri$ or $\Lo$ plus at least half of the other interval. Thus, if $\sigma$ 
     is such that $\E_{Q_\sigma}(h) \ge \left(\frac{1}{c_1 n_P}\right)^{1/\rho_a} $, then 
     $\E_{Q_\sigma}(\hat{h}) \ge 1/2 \cdot \left(\frac{1}{c_1 n_P}\right)^{1/\rho_a}$.
 \end{proof}
    
We define the event $B$ and randomize the choice of $\sigma$ as in the proof of \Cref{thm: simple rho lower bound }. The constructions are such that the event $B$ has exactly the same probability as in the proof of \Cref{thm: simple rho lower bound }, and the rest of the proof follows by exactly the same argument.   
\end{proofof}

\end{document}